\documentclass[10pt]{article}
\usepackage[preprint]{tmlr}

\usepackage{amsmath,amsfonts,bm}

\def\eqref#1{equation~\ref{#1}}

\def\1{\bm{1}}

\DeclareMathAlphabet{\mathsfit}{\encodingdefault}{\sfdefault}{m}{sl}
\SetMathAlphabet{\mathsfit}{bold}{\encodingdefault}{\sfdefault}{bx}{n}

\usepackage{stmaryrd}

\usepackage[utf8]{inputenc} 
\usepackage[T1]{fontenc}    
\usepackage{hyperref}       
\usepackage{url}            
\usepackage{booktabs}       
\usepackage{amsfonts}       
\usepackage{nicefrac}       
\usepackage{microtype}      
\usepackage{xcolor}         
\usepackage{amsmath}
\usepackage{mathtools}
\usepackage{amsthm}
\usepackage{bm}
\usepackage{caption}
\usepackage{imakeidx}
\usepackage{multirow}
\usepackage{wrapfig}

\newtheorem{theorem}{Theorem}[section]

\newtheorem{remark}{Remark}[section]

\numberwithin{equation}{section}
\allowdisplaybreaks

\usepackage{colortbl}
\definecolor{tablegray}{gray}{0.9}
\usepackage{algorithm}
\usepackage{algpseudocode}

\usepackage{amsmath}
\usepackage{tikz}
\usepackage[scaled]{helvet}
\usetikzlibrary{arrows.meta,positioning,calc,backgrounds}
\definecolor{wblue}{HTML}{0000FF}
\definecolor{ablue}{HTML}{001DBC}
\definecolor{palDeep}{HTML}{522D60}   
\definecolor{palPlum}{HTML}{80426E}   
\definecolor{palGold}{HTML}{F2A67B}  
\definecolor{palGold}{HTML}{FFD38B}   
\newcommand{\dpx}{0.02cm}
\usepackage{makecell}
\usepackage{graphicx}
\usepackage{pifont}

\newcommand{\cmark}{\textcolor{green!70!black}{\ding{51}}}
\newcommand{\xmark}{\textcolor{red}{\ding{55}}}

\title{SpectrumQuant: Spectrum Alignment for LLM Quantization}

\author{Dwaipayan Haldar\thanks{These authors contributed equally to this work.}\\
   Indian Institute of Science\\
   Bengaluru-560012, Karnataka, India \\
   \texttt{dwaipayanh@iisc.ac.in} 
   \AND
   Indranil Patra\footnote[1] \qquad\\
   Indian Institute of Science\\
   Bengaluru-560012, Karnataka, India \\
  \texttt{indranilp@iisc.ac.in} 
   \AND
   Azhar Yousuf\footnote[1] \qquad\\
   Indian Institute of Science\\
   Bengaluru-560012, Karnataka, India \\
  \texttt{azharyousuf@iisc.ac.in} 
  \AND
   Chandra Sekhar Seelamantula \\
  Indian Institute of Science \\
  Bengaluru-560012, Karnataka, India \\
  \texttt{css@iisc.ac.in} \\
}

\begin{document}

\maketitle

\begin{abstract}
Orthogonal transforms have emerged as an effective tool for quantization of large language models (LLMs), as they reshape activation distributions without significantly affecting the functionality. Existing approaches typically rely either on data-agnostic transforms, such as Hadamard rotations, or on calibration-time optimization of learned rotations. We study rotation selection from first principles and show that maximizing covariance leads to spectrum alignment and provides a principled criterion for constructing quantization-friendly orthogonal transforms. Based on this insight, we propose SpectrumQuant as a method for post-training quantization of LLMs. SpectrumQuant relies on a COvariance-Aware orthogonal Transform (COAT) derived in \textit{closed-form} that estimates activation covariance from calibration data. Our formulation characterizes quantization-friendly transforms as compositions of alignment and mixing components, with the Hadamard and butterfly transforms as special cases. As such, SpectrumQuant bridges the gap between data-agnostic and learned rotation approaches, providing an efficient data-aware alternative for post-training quantization. Experiments on Llama and Mistral models show that SpectrumQuant is competitive with state-of-the-art rotation approaches while requiring less than one-third of the GPU memory and significantly less time than learned rotation approaches like SpinQuant and DartQuant.
\end{abstract}

\section{Introduction}
Large language models (LLMs) have demonstrated exceptional proficiency across a broad range of natural language processing applications, driving rapid growth in model scale. While this scaling has enabled strong capabilities, it has also created a major deployment bottleneck due to the high cost of inference. To improve inference efficiency, a variety of approaches have been explored, including pruning, distillation, and quantization  \citep{han2016deepcompressioncompressingdeep,hinton2015distillingknowledgeneuralnetwork,pruning, yao2022zeroquantefficientaffordableposttraining}. Among these, post-training quantization (PTQ) is particularly attractive because it reduces memory usage and inference cost by mapping weights and activations to low-precision representations, without requiring expensive retraining or fine-tuning. However, activation outliers remain a fundamental obstacle to accurate low-bit quantization in LLMs. Because a small number of extreme values can dominate the activation dynamic range, they substantially degrade low-bit quantization accuracy, motivating approaches that explicitly reshape activations into more quantization-friendly representations.

Activation quantization is typically more challenging than weight quantization in LLMs because extreme outliers can substantially distort the quantization range and degrade low-bit accuracy. Addressing this challenge, previous work has proposed several outlier-mitigation techniques, including mixed-precision models \citep{zhao2024atomlowbitquantizationefficient} and shifting the quantization burden from activations to weights \citep{lin2026awqactivationawareweightquantization,smoothquant}. Very recently, rotation-based approaches have proved to be a particularly effective approach \citep{chee2023quip, tseng2024quiphash}. Such approaches rely on rotational invariance, which allows rotation matrices to be integrated into nearby weights without altering the original network outputs. The resulting rotated basis disperses outlier magnitude across dimensions, leading to smoother activations and improved quantizability \citep{quarot,liu2025spinquantllmquantizationlearned,shao2026dartquant}.

Despite the empirical effectiveness of rotation-based quantization approaches, current approaches exhibit a clear trade-off. Fixed rotations such as the Hadamard transform
\citep{chee2023quip, quarot, tseng2024quiphash} are simple as well as efficient, but data-agnostic:
they apply the same transform to every model, even though activation outliers differ across
models and layers. On the other hand, learned rotations \citep{hu2025ostquantrefininglargelanguage, liu2025spinquantllmquantizationlearned,
shao2026dartquant, xu2026butterflyquantultralowbitllmquantization} adapt to the activation
statistics of each model, but they rely on iterative optimization, which is computationally expensive and lacks theoretical guarantees. Nevertheless,  learned rotation approaches require specialized procedures to enforce orthogonality constraints, and the underlying search-based approximation does not assure convergence to a global optimum. Additionally, such learned rotations introduce a substantial calibration cost in terms of memory and time requirement. Calibration costs are difficult to reconcile with the practical appeal of post-training quantization, whose main advantage lies in enabling fast and lightweight deployment. 

To address this gap, we propose SpectrumQuant, a data-aware and mathematically grounded approach to LLM post-training quantization. Instead of relying on either data-agnostic rotations or empirically learned rotations with no theoretical guarantees, SpectrumQuant derives quantization-friendly orthogonal transforms from first principles via a well-motivated covariance maximization objective. The resulting formulation yields a closed-form COvariance Aware orthogonal Transform (COAT) that combines a Hadamard backbone (mixing component) with an adaptive (alignment) component estimated from calibration data. Experiments across multiple prominent LLM architectures demonstrate that the proposed transform \emph{dramatically reduces the calibration costs} while yielding on-par quantization performance. Our contributions are two-fold:
\begin{itemize}
    \item We formulate a novel, well-motivated \textbf{covariance maximization objective} for rotation design in the realm of LLM post-training quantization. Following a first-principles approach, we derive a family of \textbf{closed-form} orthogonal transforms that unify fixed or data-independent structured rotations and data-dependent rotations within a single framework. 

    \item The proposed quantization method, termed SpectrumQuant, delivers high accuracy while reducing calibration cost by a significant margin. For instance, for the 70B model, SpectrumQuant completes calibration in 43 minutes on a single A6000 GPU with pre-saved activations. It achieves a \textbf{21$\times$ speedup} in total calibration time, and a \textbf{9.8$\times$ reduction in memory usage} relative to the state-of-the-art (SOTA) approaches like SpinQuant \citep{liu2025spinquantllmquantizationlearned} while giving on-par or better accuracy. Notably, SpectrumQuant can also be \textbf{calibrated efficiently on the CPU} unlike prior learned rotation approaches such as SpinQuant.
    
\end{itemize}

\section{Related Work}

Existing works employ both non-orthogonal and orthogonal transforms for post-training quantization of LLMs, which we briefly review below.

\subsection{Non-orthogonal Transforms}
Scaling-based transforms such as SmoothQuant \citep{smoothquant} address activation outliers by shifting the outliers from activations to weights, reducing activation quantization errors. OmniQuant \citep{shao2024omniquant} proposes learnable weight clipping and uses block-wise error minimization to fine-tune quantization errors. Although these approaches reduce activation quantization errors, they shift the difficulty to weight quantization, creating problems in the presence of extreme activation outliers. Recent approaches, such as WUSH \citep{wush} and CAT \citep{federici2026dissecting}, focus on weight and activation quantization as a pair to reduce the quantization error. However, unlike orthogonal transforms, non-orthogonal transforms cannot be combined with the preceding normalization layer, so they must be applied online and introduce additional inference overhead.

\subsection{Orthogonal Transforms}

Among rotation-based quantization approaches, QuIP \citep{chee2023quip} proposes incoherent processing for reducing the impact of outliers in both the weight and activation spaces. As an extension, QuIP\# \citep{tseng2024quiphash} improves speed by using random Hadamard transforms. Building on these approaches, which are primarily meant for weight-only quantization, QuaRot \citep{quarot} proposes activation and KV-cache quantization method by using random Hadamard rotations and incoherence processing to reduce outliers. 

SpinQuant \citep{liu2025spinquantllmquantizationlearned} demonstrates that certain random rotations significantly outperform others in LLM quantization, and introduces an approach that optimizes rotation matrices to maximize network accuracy. Unlike SpinQuant,OSTQuant \citep{hu2025ostquantrefininglargelanguage} treats both rotation and scaling matrices as network parameters and fine-tunes them in an end-to-end fashion, resulting in the transformations becoming non-orthogonal. In contrast, DartQuant \citep{shao2026dartquant} optimizes orthogonal rotations through a calibration-time loss, making it a calibration-based alternative rather than an end-to-end fine-tuned one. ButterflyQuant \citep{xu2026butterflyquantultralowbitllmquantization} finds the parameters of butterfly matrices via training, thereby reducing the calibration cost. 
In a nutshell, existing orthogonal-transform-based approaches can be broadly divided into two classes: data-agnostic random transforms, such as Hadamard-based approaches \citep{chee2023quip,quarot,tseng2024quiphash}, and data-dependent optimized transforms, either through end-to-end fine-tuning as in SpinQuant or through calibration-time loss optimization as in DartQuant. 
In contrast to these approaches, COAT \citep{patracoat} introduces a correlation-aware orthogonal transform and demonstrates application to LLM quantization. The proposed SpectrumQuant strikes a balance by adopting a first-principles approach to seek a \textit{data-aware orthogonal transform} and constructing it in \textit{closed-form} from calibration statistics, without iterative optimization. A qualitative characterization of some prominent LLM quantization approaches in terms of the nature of the underlying rotation transform (data-agnostic/learned/non-orthogonal/etc.) is given below in Table~\ref{tab:specificaiton}.

\begin{table*}[h]
\caption{Qualitative characterization of different LLM quantization approaches.}
\label{tab:specificaiton}
\centering
\resizebox{\textwidth}{!}{%
\begin{tabular}{l|c|c|c|c|c}
\toprule
& \makecell{Data-Agnostic \\ Rotation \\ \scriptsize{(QuaRot, QuIP\#)}} 
& \makecell{Learned \\ Rotation \\ \scriptsize{(SpinQuant, OSTQuant)}} 
& \makecell{Calibration-Optimized \\ Rotation \\ \scriptsize{(DartQuant, ButterflyQuant)}} 
& \makecell{Non-Orthogonal \\ Transform \\ \scriptsize{(WUSH, CAT)}} 
& \makecell{\textbf{SpectrumQuant} \\ \textbf{(Ours)}} \\
\midrule
Closed-form & \cmark & \xmark & \xmark & \cmark & \cmark \\
Data-aware & \xmark & \cmark & \cmark & \cmark & \cmark \\
Zero inference overhead & \cmark & \cmark & \cmark & \xmark & \cmark \\
\bottomrule
\end{tabular}%
}
\end{table*}

\section{Motivation}
The proponents of SpinQuant \citep{liu2025spinquantllmquantizationlearned} posed the 
following question: \emph{``Given an activation distribution with known outlier axes and magnitudes, can a closed-form solution for the optimal rotation matrix that evenly distributes magnitude across different axes be derived}?'' In this paper, we show that a covariance-maximization objective leads to a principled derivation of the optimal rotation matrix in closed-form. The initial intuition for the covariance-maximization approach came from a toy example in the two-dimensional case as depicted in Figure \ref{fig:transformation_illustration}. 
\begin{wrapfigure}{r}{0.45\textwidth}
    \centering
    \includegraphics[trim={0cm 0cm 0cm 0.27cm}, clip, width=\linewidth]{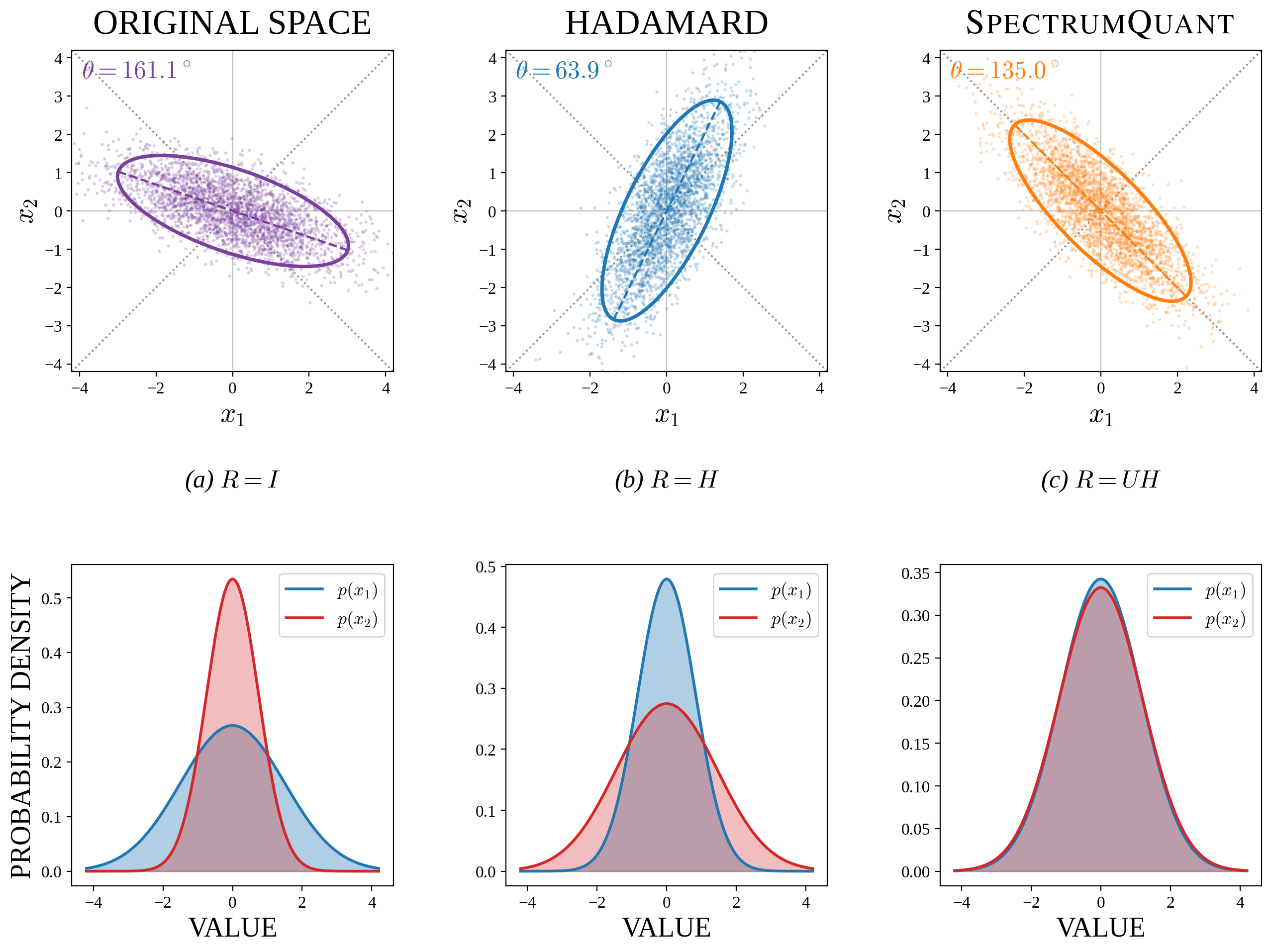}
    \caption{
    Left to right: The top row shows distribution in original space, Hadamard-transformed space, and proposed transform space. The bottom row shows corresponding marginals.}
    \label{fig:transformation_illustration}
\end{wrapfigure}
Under shared-scale quantization, the step size is dictated by the highest-variance coordinate, so a low-variance coordinate is coarsely binned and its values collapse to zero. Ideally, the data should be distributed so that it has equal variance in both coordinates, that is, its principal axis should be inclined along 45$^\circ$ (or, equivalently, 135$^\circ$). Since the 2D Hadamard matrix is a reflection transformation across the line inclined at $22.5^{\circ}$, it transforms a principal axis at angle $\varphi$ to $45^{\circ} - \varphi$. Thus, the Hadamard matrix achieves the ideal $45^{\circ}$ alignment only if the data is already axis-aligned ($\varphi = 0^{\circ}$)\footnote{The normalized Hadamard matrix $\bm{H}_2 = \frac{1}{\sqrt{2}}\begin{bmatrix} 1 & \phantom{-}1 \\ 1 & -1 \end{bmatrix}$ has the form of a reflection matrix $\begin{bmatrix} \cos 2\alpha & \phantom{-}\sin 2\alpha \\ \sin 2\alpha & -\cos 2\alpha \end{bmatrix}$ with $2\alpha = 45^{\circ}$, that is, a reflection across the line through the origin at angle $\alpha = 22.5^{\circ}$.}. This motivates 
first aligning the data using $\bm{U}$, which is the eigenbasis of the covariance matrix or simply the Karhunen–Loève Transform (KLT) \citep{klt}, before applying the Hadamard transform $\bm{H}$, resulting in a perfect $45^{\circ}$ mixing. The above description justifies a transform composed of an alignment component and a mixing component, which is precisely the idea behind our proposed SpectrumQuant. The theoretical foundation of SpectrumQuant is further supported by our analysis that covariance-maximization over the Stiefel manifold yields a variance-equalizing transform, which is shown to be optimal under shared scale quantization (detailed explanation can be found in Section~\ref{sec:right_objective}).

\section{SpectrumQuant}

We introduce SpectrumQuant as a novel quantization method derived from first principles. SpectrumQuant relies on a COvariance-Aware orthogonal Transform (COAT) derived in \textit{closed-form} that estimates activation covariance from calibration data. The closed-form formulation distinguishes SpectrumQuant from the prior art, where the transformations are either data-agnostic or learned using gradient descent optimization.

\subsection{Covariance Maximization Formulation} 
\label{section: covariance}

Let the data matrix be denoted by $\bm{X}\in\mathbb R^{n\times d}$, where $n$ is the number of data vectors and $d$ is the dimension of the data. Each row of $\bm{X}$ is a data vector. Without loss of generality, assume that the data is mean-centred. The sample covariance matrix is given by $\bm{\Sigma}_{\bm X}:=\tfrac{1}{n}\bm{X}^{\top}\bm{X}$, where $\bm X^\top$ denotes the transpose of $\bm X$. Consider the transformation $\bm{Y}=\bm{X}\bm{Q}$, where $\bm{Q}\in\mathbb R^{d\times d}$ is an orthogonal matrix. Evidently, $\bm Y$  is also mean-centred. The covariance matrix $\bm{\Sigma}_{\bm{Y}}(\bm{Q})$ of the transformed data $\bm{Y}$ is given by 
\begin{align}\label{eq:cov_transformed_data}
\bm{\Sigma}_{\bm{Y}}(\bm{Q}) \coloneq \dfrac{1}{n}\bm{Y}^{\top}\bm{Y} = \bm{Q}^{\top}\left(\dfrac{1}{n}\bm{X}^{\top}\bm{X}\right)\bm{Q}=\bm{Q}^{\top} \bm{\Sigma}_{\bm X} \bm{Q}.
\end{align}
Since our goal is to maximize the off-diagonal elements of the matrix $\bm{\Sigma}_{\bm{Y}}(\bm{Q})=\bm{Q}^{\top} \bm{\Sigma}_{\bm X} \bm{Q}$, we consider the following optimization problem over a Stiefel manifold \citep{edelman1998geometry}: 
\begin{align}\label{eq:optimization_problem}
J(\bm{Q}) \coloneq \max_{\bm{Q}} \left(\sum_{i\ne j} \big(\bm{\Sigma}_{\bm{Y}}(\bm{Q})\big)_{ij}^2\right),\quad \text{subject to}\quad \bm{Q}^{\top} \bm{Q} = \bm{I},
\end{align}
where  $\bm{I}\in\mathbb R^{d\times d}$ denotes the identity matrix. We convert the above maximization problem into a minimization by invoking invariance of the Frobenius norm to orthogonal transformations. Let $\left\|\bm{\Sigma}_{\bm{Y}}(\bm{Q})\right\|_\text{F}$ denote the Frobenius norm of the covariance matrix $\bm{\Sigma}_{\bm{Y}}(\bm{Q})$ defined via. 
\begin{align}\label{eq:Frobenius_norm}
\left\|\bm{\Sigma}_{\bm{Y}}(\bm{Q})\right\|_\text{F}^2=\sum_{i,j}\big(\bm{\Sigma}_{\bm{Y}}(\bm{Q})\big)_{ij}^2 = \sum_{i\ne j} \big(\bm{\Sigma}_{\bm{Y}}(\bm{Q})\big)_{ij}^2 
+ \sum_{i}\big(\bm{\Sigma}_{\bm{Y}}(\bm{Q})\big)_{ii}^2.
\end{align}
The optimization problem in \eqref{eq:optimization_problem} can be cast equivalently as follows:
\begin{align}\label{eq:final_optimization_problem} 
J(\bm{Q})\coloneq \max_{\bm{Q}} \left(\left\|\bm{\Sigma}_{\bm{Y}}(\bm{Q})\right\|_\text{F}^2-\sum_{i=1}^{d}\big(\bm{\Sigma}_{\bm{Y}}(\bm{Q})\big)_{ii}^2\right)
\equiv \min_{\bm{Q}} \left(\sum_{i=1}^{d}\big(\bm{\Sigma}_{\bm{Y}}(\bm{Q})\big)_{ii}^2\right),\quad \bm{Q}^{\top} \bm{Q} = \bm{I},
\end{align}
because the term $\left\|\bm{\Sigma}_{\bm{Y}}(\bm{Q})\right\|_\text{F}^2$ is  a constant offset and does not affect the optimization. That $\left\|\bm{\Sigma}_{\bm{Y}}(\bm{Q})\right\|_\text{F}^2$ is invariant to $\bm Q$ can  be verified by expressing the Frobenius norm in terms of the trace ($\bm{Tr}$) as follows: $\left\|\bm{\Sigma}_{\bm{Y}}(\bm{Q})\right\|_\text{F}^2=\bm{Tr}\left(\bm{\Sigma}_{\bm{Y}}(\bm{Q})^{\top}\bm{\Sigma}_{\bm{Y}}(\bm{Q})\right) = \bm{Tr}\left(\bm{Q}^{\top}\bm{\Sigma}_{\bm{X}}^\top\bm{\Sigma}_{\bm{X}}\bm{Q}\right) = \bm{Tr}\left(\bm{\Sigma}_{\bm{X}}^\top\bm{\Sigma}_{\bm{X}}\bm{Q}\bm{Q}^{\top}\right) = \bm{Tr}\left(\bm{\Sigma}_{\bm{X}}^\top\bm{\Sigma}_{\bm{X}}\right)$, and invoking the cyclic property of the trace operator. Hence, \eqref{eq:final_optimization_problem} serves as the desired minimization problem equivalent to the maximization problem in \eqref{eq:optimization_problem}.

\subsection{Solution Space}
We show that the objective $J(\bm{Q})$ in \eqref{eq:final_optimization_problem} is optimized by a variance-equalizing orthogonal transformation $\bm{Q}$. Since $\bm{Q}$ is orthogonal, we can always write $\bm{Q} = \bm{U}\bm{V}$, where $\bm{U}$ serves as the \emph{alignment} component which is the eigenvector matrix of $\bm{\Sigma}_{\bm{X}}$, and $\bm{V}$ is an appropriate \emph{mixing} transform. We demonstrate that there exists a class of such mixing components leading to a family of orthogonal transforms $\bm{Q} = \bm{U}\bm{V}$.

\begin{theorem}
\label{prop:minimization}
The minimization objective in \eqref{eq:final_optimization_problem} has a lower bound, which is attained if and only if the diagonal elements of $\bm{\Sigma}_{\bm{Y}}(\bm{Q})$ are all equal to the average of the eigenvalues of the covariance matrix $\bm{\Sigma}_{\bm{X}}$. Moreover, for every eigenvalue spectrum $\lambda_1 \ge \lambda_2 \ge \cdots \ge \lambda_d \ge 0$, the orthogonal transform $\bm{Q} = \bm{U}\bm{V}$ attains the bound if and only if $\bm V$ is a normalized Hadamard matrix.
\end{theorem}

\begin{proof}
Two vectors $\bm{\alpha}=\left(\alpha_1,\alpha_2,\alpha_3,\dots,\alpha_d\right)\in\mathbb R^d$ and $\bm{\beta}=\left(\beta_1,\beta_2,\beta_3,\dots,\beta_d\right)\in\mathbb R^d$, satisfy the Cauchy-Schwartz inequality: $\left(\sum_{i=1}^{d} \alpha_{i}\beta_{i}\right)^2\le \left(\sum_{i=1}^{d} \alpha_{i}^2\right) \left(\sum_{i=1}^{d} \beta_{i}^2\right),$
with the equality holding if and only if $\bm{\alpha}=\kappa \bm{\beta}$, for some scalar $\kappa$. Setting $\alpha_i = \bm{\Sigma}_{\bm{Y}}(\bm{Q})_{ii}$ and $\beta_i = 1$, gives us the inequality
\begin{align}\label{eq:lower_bounding}
\sum_{i=1}^{d}\big(\bm{\Sigma}_{\bm{Y}}(\bm{Q})\big)_{ii}^2 \ge \dfrac{1}{d} \left(\sum_{i=1}^{d} \big(\bm{\Sigma}_{\bm{Y}}(\bm{Q})\big)_{ii}\right)^2,
\end{align}
with the equality holding if and only if the diagonal elements of $\bm{\Sigma}_{\bm{Y}}(\bm{Q})$ are all equal, that is, $\big(\bm{\Sigma}_{\bm{Y}}(\bm{Q})\big)_{11}=\big(\bm{\Sigma}_{\bm{Y}}(\bm{Q})\big)_{22}=\big(\bm{\Sigma}_{\bm{Y}}(\bm{Q})\big)_{33}=\dots=\big(\bm{\Sigma}_{\bm{Y}}(\bm{Q})\big)_{dd}$. In other words, the minimum in \eqref{eq:final_optimization_problem} is achieved for an orthogonal matrix $\bm{Q}$ satisfying the condition:
\begin{align}\label{eq:optimum_mat_full}
\dfrac{1}{d} \left(\sum_{i=1}^{d} \big(\bm{\Sigma}_{\bm{Y}}(\bm{Q})\big)_{ii}\right)^2=\dfrac{1}{d}\Big(d\,\big(\bm{\Sigma}_{\bm{Y}}(\bm{Q})\big)_{ii}\Big)^2.
\end{align}
For orthogonal $\bm{Q}$, by virtue of the cyclic property of the trace operator, $\bm{Tr}\left(\bm{\Sigma}_{\bm{Y}}(\bm{Q})\right) = 
\bm{Tr}\left(\bm{Q}^{\top} \bm{\Sigma}_{\bm X} \bm{Q}\right)= \bm{Tr}\left(\bm{\Sigma}_{\bm X}\bm{Q}\bm{Q}^{\top}\right) = \bm{Tr}\left(\bm{\Sigma}_{\bm X}\right)$. Let $\lambda_i$ denote the $i^{th}$ eigenvalue of $\bm{\Sigma}_{\bm{X}}$. Therefore, \ref{eq:optimum_mat_full} implies that the optimal $\bm{Q}$ satisfies
\begin{align}
\label{eq:optimum_mat}
\quad \big(\bm{\Sigma}_{\bm{Y}}(\bm{Q})\big)_{ii}=\dfrac{1}{d}\left(\sum_{i=1}^{d}\lambda_{i}\right)=\bar{\lambda}, \quad i \in \llbracket 1,d \rrbracket,
\end{align}
which establishes variance-equalization.

Furthermore, using the eigen-decomposition of the covariance matrix: $\bm{\Sigma}_{\bm X} = \bm{U}\bm{\Lambda}\bm{U}^{\top}$, where $\bm{\Lambda} = \mathrm{diag}(\lambda_1, \lambda_2, \dots, \lambda_d)$ is a diagonal matrix of the eigenvalues of $\bm{\Sigma}_{\bm X}$, we observe that $\bm{\Sigma}_{\bm{Y}}(\bm{Q})=\bm{Q}^{\top} \bm{\Sigma}_{\bm X} \bm{Q}=\bm{V}^{\top} \bm{U}^{\top}\bm{\Sigma}_{\bm X}\bm{U}\bm{V}=\bm{V}^{\top}\bm{\Lambda}\bm{V}$. Therefore, we can write \eqref{eq:optimum_mat} as follows:
\begin{align}\label{eq:(desired_cond_on_V)}
\bm{\Sigma}_{\bm{Y}}(\bm{Q})_{ii}=\left(\bm{V}^{\top}\bm{\Lambda}\bm{V}\right)_{ii} = \sum_{k=1}^{d}\lambda_{k} v_{ki}^2 = \overline{\lambda}.
\end{align}
Since equation~\ref{eq:(desired_cond_on_V)} must hold for every choice of the eigenvalue spectrum, choosing $\bm e_k$ as a spectrum, where $\bm e_k$ is the $k^{th}$ canonical basis of $\mathbb{R}^d$, gives $v_{ki} = \pm \frac{1}{\sqrt{d}}$. Given the orthogonality of $\bm V$, we infer that $\bm V$ is a normalized Hadamard matrix. Conversely, \eqref{eq:(desired_cond_on_V)} holds for any normalized Hadamard matrix, which can be verified by substitution.
\end{proof}

From Theorem~\ref{prop:minimization}, we note that the proposed covariance maximization objective in equation~\ref{eq:final_optimization_problem} admits multiple solutions determined by different choices of the alignment component $\bm{V}$ in $\bm{Q} = \bm{U}\bm{V}$ (see Figure~\ref{fig:solution_space}). This is also evident by first observing that the number of degrees of freedom in describing a $d$-dimensional orthogonal matrix is $d(d-1)/2$; because, out of the $d^2$ entries, the unit length of the column vectors imposes $d$ constraints, and orthogonality of these column vectors imposes an additional $d(d-1)/2$ constraints, so $d^2-d-d(d-1)/2=d(d-1)/2$. Moreover, in view of Theorem~\ref{prop:minimization}, the orthogonal matrix $\bm{Q}$ must satisfy $\bm{Q}^{\top} \bm{\Sigma}_{\bm X} \bm{Q} = \bar{\lambda}$. Since $\bm{Q}$ preserves the trace of the covariance matrix, fixing the first $d-1$ elements to $\bar{\lambda}$ automatically forces the $d$-th element to $\bar{\lambda}$. Therefore, the total number of degrees of freedom in choosing the desired $\bm{Q}$ is  $d(d-1)/2-(d-1)=(d-1)(d-2)/2$.

\subsubsection{Hadamard Solutions}
\label{sec:all_sols}
From Theorem~\ref{prop:minimization}, the class of Hadamard matrices\footnote{A Hadamard matrix is an orthogonal matrix whose entries are either $+1$ or $-1$. We emphasize that though the general Haar matrix of dimension $d\ge 2$ is also an orthogonal matrix with entries $+1$, $-1$, or $0$, it doesn't constitute a solution to our problem following Theorem~\ref{prop:minimization}. Nonetheless, a Haar matrix of dimension $d=2$ coincides with the Hadamard matrix.} constitutes the solution space for $\bm V$ in the orthogonal transform $\bm{Q = UV}$. However, the existence of Hadamard matrix of any dimension is an open problem \citep{paley1933orthogonal}. Several approaches have been proposed for constructing Hadamard matrices, including Sylvester's construction for dimensions that can be expressed as a power-of-$2$ \citep{sylvester1867lx}, construction by Kronecker products for certain dimensions that can be expressed as multiples of 4 \citep{seberry2020hadamard, hedayat1978hadamard}, etc. The class of butterfly matrices offers a more general framework for constructing Hadamard matrices that subsumes Sylvester's construction \citep{dao2019learning, trogdon2019spectral}, as depicted in Figure~\ref{fig:solution_space}. 
\begin{figure}
\centering
\resizebox{0.5\linewidth}{!}{%
\begin{tikzpicture}[font=\small, >=stealth]
  \definecolor{palDeep}{HTML}{522D60}
  \definecolor{palPlum}{HTML}{80426E}
  \definecolor{palPeach}{HTML}{F2A67B}
  \definecolor{palGold}{HTML}{FFD38B}
  \fill[palPeach!35!white] (3.6,0) ellipse (3.4 and 1.5);
  \fill[palPlum!18!white] (0,0) ellipse (2.2 and 1.8);
  \begin{scope}
    \clip (0,0) ellipse (2.2 and 1.8);
    \fill[palGold!70!white] (3.6,0) ellipse (3.4 and 1.5);
  \end{scope}
  \draw[thick, palPeach!80!palPlum] (3.6,0) ellipse (3.4 and 1.5);
  \draw[thick, palPlum] (0,0) ellipse (2.2 and 1.8);
  \draw[thick, palDeep, fill=palPlum!45!white] (1.3,-0.1) ellipse (0.70 and 0.60);
  \node[text=palDeep] at (1.3,-0.1) {\small Sylvester};
  \node[align=center, text=palDeep] at (-0.6,2.15)
    {\MakeUppercase{Hadamard matrices}};

  \node[align=center, text=palDeep] at (5.4,1.9)
    {\MakeUppercase{Butterfly matrices}, $d=2^\ell$\\[-1pt]
     \scriptsize (Arbitrary angles $\theta$)};

  \node[align=center, text=palPlum] at (-1.1,0) {$d\neq 2^\ell$};

  \node[align=center, text=palPlum] at (5.0,-0.1) {General $\theta$};

  \draw[->, thick, palDeep] (1.75,0.95) -- (2.4,2.55)
    node[above, align=center, text=palDeep]
    {$\theta\in\big\{\tfrac{\pi}{4},\tfrac{3\pi}{4},\tfrac{5\pi}{4},\tfrac{7\pi}{4}\big\}$\\
     $\mathcal O(d\log d)$ solutions};
\end{tikzpicture}
}
\caption{Structure of the solution space of the \emph{mixing} component ${\bm V}$ in our COAT $\bm{Q}=\bm{U}\bm{V}$. By Theorem~\ref{prop:minimization}, the data-independent mixing components are exactly the normalized Hadamard matrices. The Butterfly matrices with rotation angles chosen from $\{\pi/4, 3\pi/4, 5\pi/4, 7\pi/4\}$ are a subset of solutions; this subset admits
$\mathcal O(d\log d)$ operations and contains the class of Sylvester-Hadamard matrices.}
\label{fig:solution_space}
\end{figure}
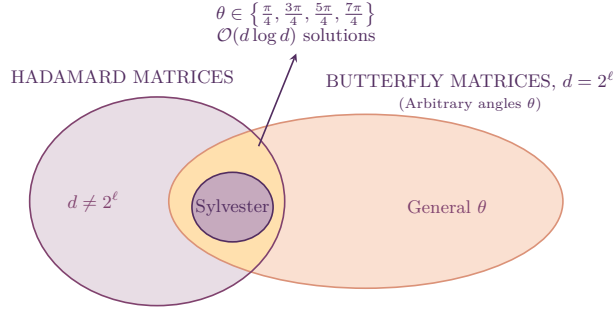

Sylvester's construction starts with $\bm{H}_{1}=[1]$, 
$\bm{H}_{2}=\dfrac{1}{\sqrt{2}}
\begin{bmatrix}
   1 & \phantom{-}1\\
   1 & -1
\end{bmatrix}$ 
and $\bm{H}_{2^{\ell}} = \dfrac{1}{\sqrt{2}}\left(\bm{H}_{2} \otimes \bm{H}_{2^{\ell-1}}\right)$, where $\ell=2,3,\dots$ and $\otimes$ denotes the Kronecker product. Sylvester's solution belongs in the family of \emph{Butterfly} matrices obtained via the ``butterfly computations'' \citep{dao2019learning, trogdon2019spectral}. A butterfly computation is a kind of mixing operation that takes two inputs $x_i$ and $x_j$ and gives two outputs $y_i=x_i + \tau\,x_j$ and $y_j=x_i -  \tau\,x_j$, where $\tau$ is an appropriately chosen twiddle factor. The computation draws its name from the fact that, in the data-flow diagram for this pair of operations, lines cross and resemble the wings of a butterfly. A butterfly matrix $\bm{B}\in \mathbb R^{d\times d}$ can be expressed as a sparse hierarchical factorization:
\begin{align}\label{eq:butterfly_matrix}
 \bm{B}=\left(\prod_{i=1}^{\log_2 d}  {\bm B}_{i}^\top\right)^\top, 
\end{align}
where each ${\bm B}_i$ is a sparse orthogonal matrix consisting of $d/2$ independent Givens rotations $\bm{G}(\theta)$ defined as 
$\begin{bmatrix}
    \cos\theta & -\sin\theta\\
    \sin\theta & \phantom{-}\cos\theta
\end{bmatrix}$ or its reflection counterpart 
$\begin{bmatrix}
    \cos\theta & \phantom{-}\sin\theta\\
    \sin\theta & -\cos\theta
\end{bmatrix}.$
More explicitly, each ${\bm B}_{i}$ in the factorization \ref{eq:butterfly_matrix} is given as follows: 
\begin{align}\label{eq:butterfly_matrix_i}
 {\bm B}_i=\bm{P}_i\, {\bm {D}_i}\,\bm{P}_i^\top,\quad \text{with}\quad {\bm D}_i=\text{diag}\big(\bm{G}(\theta_{i_1}),\bm{G}(\theta_{i_2}),\dots, \bm{G}(\theta_{i_{d/2}})\big),
\end{align}
where $\bm{P}_i$ is a permutation matrix determining the butterfly connectivity pattern by pairing indices at stage $i=1,2,\dots,\log_2 d$. Crucially, the choice of index pairings across stages is not unique: while the standard dyadic stride pattern, that is, pairing indices separated by stride $2^{i-1}$ at stage $i$, represents a canonical choice, any sequence of stage permutations 
$\{\bm{P}_i\}_{i=1}^{\log_2 d}$ whose pairing rules satisfy the unique-path property, that is, the active support of every basis vector doubles at each stage without path collisions, yields a valid butterfly matrix. Note that, while the block-diagonal matrix comprising $\text{diag}\big(\bm{G}(\theta_{i_1}),\bm{G}(\theta_{i_2}),\dots, \bm{G}(\theta_{i_{d/2}})\big)$ can only apply 2D rotations to adjacent indices/channels like $(1,2), (3,4), \dots,(d-3, d-2), (d-1,d)$, the permutation matrices $\bm{P}_i$ and $\bm{P}_i^\top$ play a structural role. The matrix ${\bm P}_i$ does reshuffling so that pairs separated by a stride pattern at stage $i$ are moved into adjacent positions $(2m-1, 2m)$, $m=1,2,3,\dots,d/2$, to be amenable to the action of the block-diagonal matrix of Givens rotations ${\bm D_i}$. 

It can be shown that by choosing the angle in each of the rotation blocks as $\theta\in \{\pi/4, 3\pi/4, 5\pi/4, 7\pi/4\}$, any butterfly matrix from \eqref{eq:butterfly_matrix} satisfies \eqref{eq:(desired_cond_on_V)}. This is because, any fixed evaluation at $\pi/4,\, 3\pi/4,\, 5\pi/4$ or  $7\pi/4$ yields a butterfly matrix with entries $\pm \frac{1}{(\sqrt{2})^{\log_2d}}$. For $d=2^\ell$, the $(k,i)$-th element of the resulting butterfly matrix is $b_{ki}:= \pm\tfrac{1}{(\sqrt{2})^{\log_2d}} = \pm\tfrac{1}{(\sqrt{2})^{\ell}}=\pm\tfrac{1}{\sqrt{d}}$. Accordingly, any such butterfly matrix satisfies the condition:
\begin{align}\label{eq:butterfly_choice}
\big(\bm{\Sigma}_{\bm{Y}}(\bm{Q})\big)_{ii}=\sum_{k=1}^{d}\lambda_{k} b_{ki}^2= \sum_{k=1}^{d}\lambda_{k} \left(\dfrac{1}{\sqrt{d}}\right)^2=\sum_{k=1}^{d}\dfrac{\lambda_{k}}{d}=\overline{\lambda}.
\end{align}
The above description demonstrates that a butterfly matrix constructed as per \eqref{eq:butterfly_matrix} is a solution for the covariance maximization objective considered in \ref{eq:final_optimization_problem}, whenever the dimension $d$ is a power of $2$, which is usually the case with LLMs. Therefore, in this case, the desired covariance-maximizing orthogonal transform leads to variance equalization is given by $\bm{Q}:=\bm{U}\bm{V}=\bm{U}\bm{B}$. 

\begin{remark}[Non-power-of-$2$ case]
As the Kronecker product of any two Hadamard matrices is again a Hadamard matrix, it is possible to construct more general Hadamard matrices. For instance, if $\bm{H}_{m}$ is a Hadamard matrix of some known dimension $m$ \citep{moregenralHadamard}, then we can construct another Hadamard matrix of dimension $d = 2^\ell m$ via the Kronecker product as $\bm{H}_d = \bm{H}_{2^\ell} \otimes \bm{H}_m$. Analogously, if $\bm{B}_{2^\ell}$ is a butterfly matrix of order $2^\ell$ with rotation angles chosen from $\{\pi/4, 3\pi/4, 5\pi/4, 7\pi/4\}$, then
$\bm{B}_d = \bm{B}_{2^\ell} \otimes \bm{H}_m$ is a (mixed-radix) butterfly matrix of dimension $d = 2^\ell m$, which is again a Hadamard matrix. In both cases, we can apply the matrix in $\mathcal O(d \ell + d\,m)$ operations without forming the Kronecker product.
\end{remark}

\subsection{Optimal Transform Under Shared-Scale Quantization}
\label{sec:right_objective}
It can be shown that among all orthogonal transforms, the total quantization error under shared-scale quantization is minimized if and only if the transformation performs \textit{variance equalization}. In the regime of shared-scale quantization, a single scale is applied per quantization group, which is often the case with modern quantization approaches. Each quantization group is assigned a single dynamic range $A = c\sigma_{max}(\bm{Q})$, where $\sigma_{max}^2(\bm{Q})= \displaystyle\max_{1\le i\le d}(\Sigma_{\bm Y}(\bm{Q}))_{ii}$ with $c$ is a constant determined by the desired coverage probability. The constant $c$ is independent of both the transformation $\bm{Q}$ and the coordinate $i=1,2,\dots, d$. For instance, as used for the $99\%$ coverage rule, this constant is specified as $c=3$. Considering the $b$-bit scalar uniform quantizer $q(\cdot)$ with step-size $\Delta = \frac{2A}{2^b}$, the quantization noise variance is $\sigma_q^2 = \Delta^2/12$ under the standard high-rate approximation. Since the scale is shared across all the coordinates, every coordinate incurs identical quantization noise variance:
\begin{align}
\sigma_{q_i}^2=\dfrac{\Delta^2}{12}&=\dfrac{1}{12}\left(\dfrac{4c^2\sigma_{max}^2(\bm{Q})}{2^{2b}}\right)=\dfrac{c^2}{3}\,2^{-2b}\,\sigma_{max}^2(\bm{Q})=c_0\,\sigma_{max}^2(\bm{Q})\,2^{-2b},\quad i=\llbracket 1, d\rrbracket,
\end{align}
where $c_0 = c^2/3$. The total quantization noise variance is given by 
$\boldsymbol{\varepsilon}(\bm{Q})= \displaystyle\sum_{i=1}^{d}\sigma_{q_i}^2=d\,c_0\,\sigma_{max}^2(\bm{Q})\,2^{-2b}.$
Observe that
\begin{align}
\label{eq:varianceequalization}
\sigma_{max}^2(\bm{Q})&= \displaystyle\max_{1\le i\le d}(\Sigma_{\bm Y}(\bm{Q}))_{ii}\ge \dfrac{\sum_{i=1}^{d}(\Sigma_Y(\bm{Q}))_{ii}}{d}=\dfrac{\bm{Tr}\left(\bm{\Sigma}_{\bm{Y}}(\bm{Q})\right)}{d}=\dfrac{\bm{Tr}\left(\bm{\Sigma}\right)}{d}=\overline{\lambda}.    
\end{align}
Therefore, it follows that $\boldsymbol{\varepsilon}(\bm{Q})\ge d\,c_0\,\overline{\lambda}\,2^{-2b}$ with equality if and only if the diagonal elements of $\bm{\Sigma}_{\bm{Y}}(\bm{Q})$ are all equal: $\big(\bm{\Sigma}_{\bm{Y}}(\bm{Q})\big)_{11}=\big(\bm{\Sigma}_{\bm{Y}}(\bm{Q})\big)_{22}=\big(\bm{\Sigma}_{\bm{Y}}(\bm{Q})\big)_{33}=\dots=\big(\bm{\Sigma}_{\bm{Y}}(\bm{Q})\big)_{dd}$. Hence, we conclude that the total quantization error $\boldsymbol{\varepsilon}(\bm{Q})$ is minimized only under a variance equalizing transformation. 

To analyze the quantization distortion under an orthogonal transformation, we consider synthetic multivariate Gaussian distributed data. We construct a covariance matrix ${\bm \Sigma}_{\bm X} = {\bm U} {\bm \Lambda} {\bm U}^\top \in \mathbb{R}^{d \times d}$, where $d=4096$, ${\bm U}$ is a random orthogonal matrix, and $\Lambda$ is a diagonal matrix with exponentially decaying eigenvalues. Ten thousand samples are drawn from ${\bm X} \sim \mathcal{N}({\bm 0},{\bm \Sigma}_{\bm X})$. We compare three quantization pipelines: quantization in the original basis; quantization after a Hadamard transform; and quantization after our proposed covariance-aware transform, viz., COAT (${\bm Q}={\bm U}{\bm H}$). For each bit-width, we uniformly quantize the transformed samples, project back to the original space and report the MSE. 

In the transformed case, the range can be chosen according to variance-based scaling: $A = 3\sqrt{\max_i \widehat{\text{Var}}(y_i)}$, where $y_i$ denotes the $i$-th component of the transformed vector ${\bm y}$. The dynamic range can also be chosen according to AbsMax scaling:  $A=\|\bm{y}\|_\infty$. The MSE versus number of bits is shown in Figure~\ref{fig:mse_distortion_scales}. The transform that minimizes the worst-case coordinate variance also corresponds to a smaller quantization step, which explains why SpectrumQuant achieves a lower distortion.

\begin{figure}[t]
  \centering
  \includegraphics[width=0.9\textwidth]{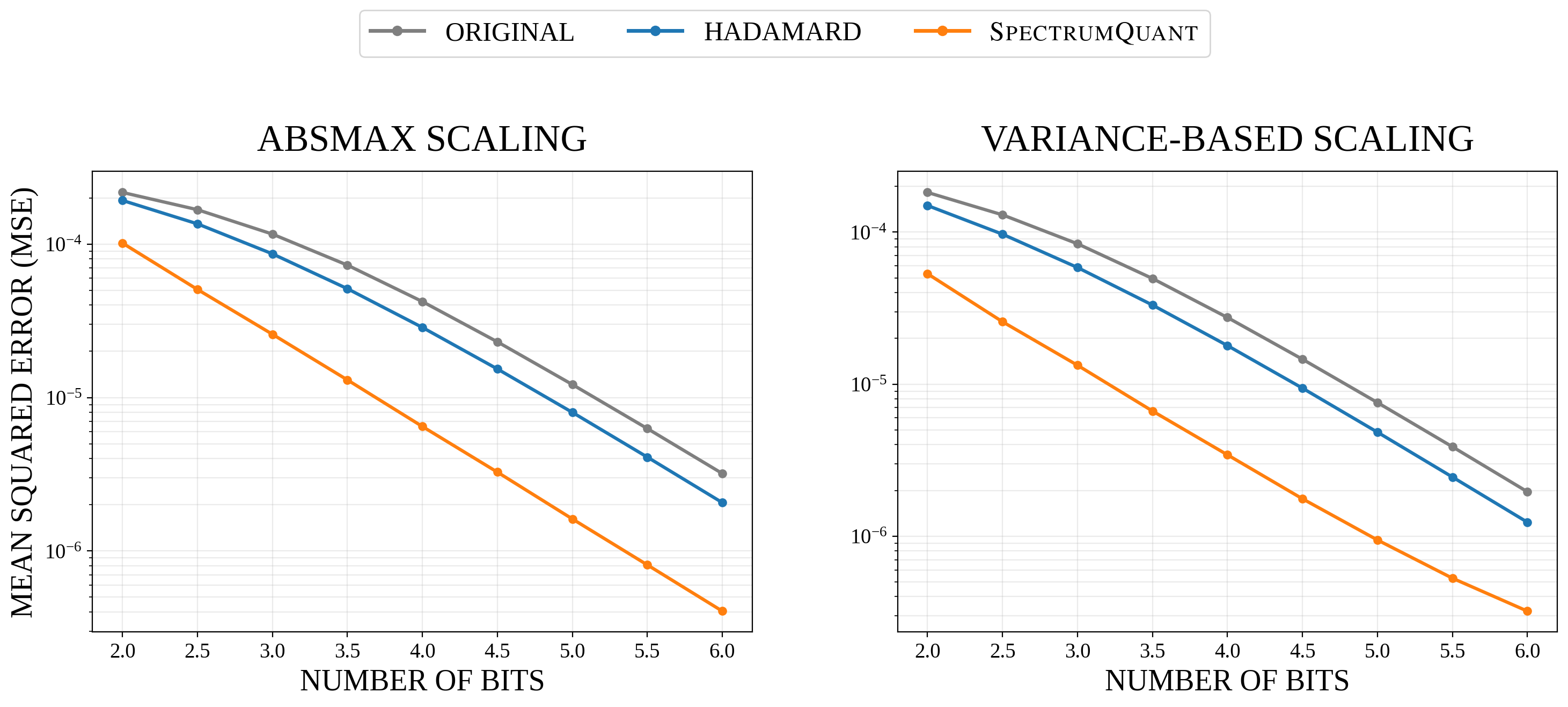}
    \caption{MSE as a function of the number of bits for a synthetic multivariate Gaussian experiment. The left panel uses AbsMax scaling $A=\|\bm{y}\|_\infty$, and the right panel uses variance-based scaling $A = 3\sqrt{\max_i \widehat{\mathrm{Var}}(y_i)}$.}
  \label{fig:mse_distortion_scales}
\end{figure}

\section{Experiments}
\label{sec:rotational-invariance}
The Transformer architecture, common in LLMs, is composed of multi-head self-attention and feed-forward network modules both composed of linear layers. Following the standard notation, denote the output of a layer as $\bm Y := \bm{XW}$, with $\bm X \in \mathbb{R}^{d_{in} \times d}$ as the input activation, and $\bm W \in \mathbb{R}^{ d \times d_{out}}$ as the weight matrix. Based on rotational invariance, an orthogonal transformation $\bm{R} \in \mathbb{R}^{d \times d}$  is inserted into the linear layer without altering the output, yielding $\bm{Y} = \bm{X R R^{\top} W}= \bm{(XR)(R^{\top}W)}$. We can commute $\bm R$ with the RMSNorm operation, because an orthogonal matrix $\bm R$ satisfies: $RMSNorm(\bm {xR}) = RMSNorm(\bm x)\bm R$, where $\bm x$ is a row vector of $\bm X$. Therefore, we can combine $\bm{R}$ with the previous weight matrix and $\bm{R}^{\top}$ with current layer's weight matrix, and subject the activations to a rotation without incurring additional inference costs.

Without altering the output of the model, a total of \emph{four orthogonal matrices}, \textit{viz.} $\bm{R}_1$, $\bm{R}_2$, $\bm{R}_3$, $\bm{R}_4$ can be inserted within the Transformer block in Fig.~\ref{fig:rotational_invariance} (adapted from  \citep{liu2025spinquantllmquantizationlearned}). Multiplying $\bm{R}_1^{\top}$ on the left side of $\bm W_q$, $\bm W_k$, $\bm W_v$, $\bm W_{up}$, $\bm W_{gate}$, and $\bm{R}_1$ on the right side of $\bm W_{o}$ and $\bm W_{down}$, an equivalent transformation is achieved. The rotation $\bm{R}_2$ can be inserted between $\bm W_v$ and $\bm W_{o}$, whereas $\bm{R}_3$ can be inserted between the rotated encodings of \textit{Query} and \textit{Key}, and $\bm{R}_4$ can be inserted before $\bm W_{down}$. Finally, $\bm W_{emb}$ is multiplied on the right by $\bm{R}_1$, and $\bm{R}_1^{\top}$  is multiplied on the left of $\bm W_{head}$, completing all equivalent transformations. This inference framework follows the  rotational transformation framework used in the baselines, including SpinQuant \citep{liu2025spinquantllmquantizationlearned} and DartQuant \citep{shao2026dartquant}.

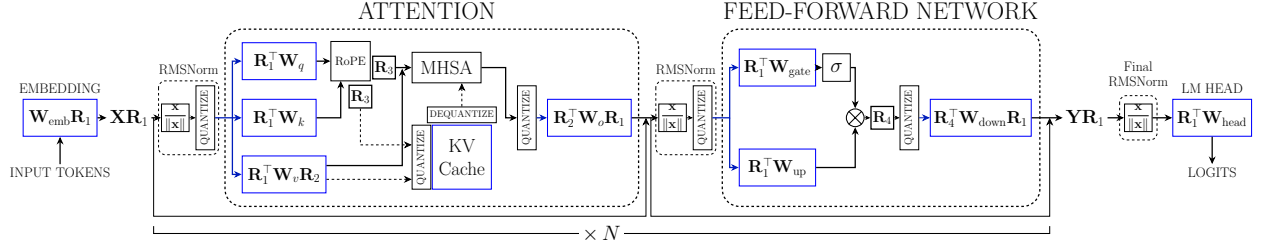
\begin{figure*}[t]
  \centering
  \resizebox{\textwidth}{!}{
\begin{tikzpicture}[
    x=\dpx, y=-\dpx,
    font=\rmfamily,
    >={Stealth[length=9pt,width=8pt]},
    fW/.style={font=\fontsize{20}{24}\selectfont},                    
    fL/.style={font=\rmfamily\fontsize{16}{19}\selectfont},  
    fT/.style={font=\rmfamily\fontsize{28}{32}\selectfont},  
    fQ/.style={font=\rmfamily\fontsize{12}{14}\selectfont},  
    weight/.style={draw=wblue, line width=1.2pt, minimum height=75*\dpx,
                   inner sep=4pt, fW},
    quant/.style={draw, line width=1pt, minimum width=38*\dpx,
                  minimum height=130*\dpx, inner sep=0pt, fQ},
    sqbox/.style={draw, line width=1pt, inner sep=0pt},
    normbox/.style={sqbox, minimum size=56*\dpx,
                    font=\fontsize{15}{17}\selectfont},
    rot/.style={draw, line width=1.5pt, minimum size=46*\dpx, inner sep=0pt,
                font=\fontsize{18}{20}\selectfont},
    blockframe/.style={draw, dashed, line width=1.5pt, rounded corners=20pt},
    normframe/.style={draw, dashed, line width=1pt, rounded corners=8pt},
    main/.style={->, line width=1.5pt},
    qa/.style={->, line width=1.5pt, draw=ablue},
    thin/.style={->, line width=1pt},
    kvdash/.style={->, line width=1pt, dashed},
  ]

\node[weight, minimum width=150*\dpx] (emb) at (190,340)
  {$\mathbf{W}_{\mathrm{emb}}\mathbf{R}_1$};
\node[fL] at (190,280) {EMBEDDING};
\node[align=center, font=\rmfamily\fontsize{16}{19}\selectfont]
  (tok) at (190,450) {INPUT TOKENS};
\draw[main] (tok) -- (emb);

\node[font=\fontsize{22}{26}\selectfont] (XR) at (330,340) {$\mathbf{X}\mathbf{R}_1$};
\draw[main] (emb) -- (XR);

\draw[normframe] (390,265) rectangle (512,415);
\node[fL] at (451,242) {RMSNorm};
\node[normbox] (n1) at (424,340) {$\dfrac{\mathbf{x}}{\|\mathbf{x}\|}$};
\node[quant] (q1) at (485,340) {\rotatebox{90}{QUANTIZE}};
\draw[main] (XR) -- (n1);
\draw[thin, >={Stealth[length=6pt,width=5pt]}] (n1) -- (q1);

\draw[blockframe] (525,160) rectangle (1365,520);
\node[fT] at (935,122) {ATTENTION};

\node[weight, minimum width=150*\dpx] (wq) at (635,225)
  {$\mathbf{R}_1^{\top}\mathbf{W}_q$};
\node[weight, minimum width=150*\dpx] (wk) at (635,340)
  {$\mathbf{R}_1^{\top}\mathbf{W}_k$};
\node[weight, minimum width=170*\dpx] (wv) at (645,455)
  {$\mathbf{R}_1^{\top}\mathbf{W}_v\mathbf{R}_2$};

\node[sqbox, minimum size=80*\dpx,
      font=\rmfamily\fontsize{14}{16}\selectfont]
  (rope) at (780,225) {RoPE};
\node[rot] (r3a) at (848,238) {$\mathbf{R}_3$};
\node[rot] (r3b) at (800,297) {$\mathbf{R}_3$};

\node[sqbox, minimum width=140*\dpx, minimum height=62*\dpx,
      font=\rmfamily\fontsize{22}{26}\selectfont]
  (mhsa) at (975,238) {MHSA};
\node[draw=wblue, line width=1.2pt, minimum width=120*\dpx,
      minimum height=130*\dpx, align=center, inner sep=0pt,
      font=\rmfamily\fontsize{22}{26}\selectfont]
  (kv) at (1006,420) {KV\\Cache};
\node[quant] (qkv) at (924,420) {\rotatebox{90}{QUANTIZE}};
\node[sqbox, minimum width=140*\dpx, minimum height=32*\dpx,
      font=\rmfamily\fontsize{11}{13}\selectfont]
  (deq) at (1006,334) {DEQUANTIZE};
\node[quant] (q2) at (1138,340) {\rotatebox{90}{QUANTIZE}};
\node[weight, minimum width=170*\dpx] (wo) at (1265,340)
  {$\mathbf{R}_2^{\top}\mathbf{W}_o\mathbf{R}_1$};

\draw[qa] (q1) -- (wk);
\draw[qa, <->] (wv.west) -- (540,455) -- (540,225) -- (wq.west);

\draw[main] (wq) -- (rope);
\draw[main] (wk.east) -- (760,340) -- (760,265);
\draw[main] (r3a) -- (r3a -| mhsa.west);
\draw[main] (wv.east |- 0,437) -- (885,437) -- (885,240);

\draw[kvdash] (r3b.south) -- (800,395) -- (qkv.west |- 0,395);
\draw[kvdash] (wv.east |- 0,465) -- (qkv.west |- 0,465);
\draw[kvdash] (deq.north) -- (deq.north |- mhsa.south);

\draw[main] (mhsa.east) -- (1095,238) -- (1095,340) -- (q2.west);
\draw[qa]   (q2) -- (wo);

\draw[normframe] (1400,265) rectangle (1522,415);
\node[fL] at (1461,242) {RMSNorm};
\node[normbox] (n2) at (1434,340) {$\dfrac{\mathbf{x}}{\|\mathbf{x}\|}$};
\node[quant] (q3) at (1495,340) {\rotatebox{90}{QUANTIZE}};
\draw[main] (wo) -- (n2);
\draw[thin, >={Stealth[length=6pt,width=5pt]}] (n2) -- (q3);

\draw[thin] (380,340) -- (380,540) -- (1380,540) -- (1380,340);

\draw[blockframe] (1535,160) rectangle (2180,520);
\node[fT] at (1857,122) {FEED-FORWARD NETWORK};

\node[weight, minimum width=155*\dpx] (wg) at (1647,238)
  {$\mathbf{R}_1^{\top}\mathbf{W}_{\mathrm{gate}}$};
\node[weight, minimum width=155*\dpx] (wu) at (1647,441)
  {$\mathbf{R}_1^{\top}\mathbf{W}_{\mathrm{up}}$};
\node[sqbox, minimum size=56*\dpx, font=\fontsize{24}{28}\selectfont]
  (sig) at (1767,238) {$\sigma$};
\node[draw, circle, line width=1.2pt, minimum size=36*\dpx, inner sep=0pt]
  (mul) at (1805,340) {};
\draw[line width=1.2pt] (mul.north east) -- (mul.south west)
                        (mul.north west) -- (mul.south east);
\node[rot] (r4) at (1860,340) {$\mathbf{R}_4$};
\node[quant] (q4) at (1915,340) {\rotatebox{90}{QUANTIZE}};
\node[weight, minimum width=205*\dpx] (wd) at (2060,340)
  {$\mathbf{R}_4^{\top}\mathbf{W}_{\mathrm{down}}\mathbf{R}_1$};

\draw[qa] (q3.east) -- (1550,340) -- (1550,238) -- (wg.west);
\draw[qa] (1550,340) -- (1550,441) -- (wu.west);
\draw[main] (wg) -- (sig);
\draw[main] (sig.east) -- (1805,238) -- (mul.north);
\draw[main] (wu.east) -- (1805,441) -- (mul.south);
\draw[thin, >={Stealth[length=6pt,width=5pt]}] (mul) -- (r4);
\draw[thin, >={Stealth[length=6pt,width=5pt]}] (r4) -- (q4);
\draw[thin, draw=ablue] (q4) -- (wd);

\node[font=\fontsize{22}{26}\selectfont] (YR) at (2272,340) {$\mathbf{Y}\mathbf{R}_1$};
\draw[main] (wd) -- (YR);

\draw[thin] (1390,340) -- (1390,540) -- (2200,540) -- (2200,340);

\draw[line width=1.2pt] (380,560) -- (380,575) -- (2200,575) -- (2200,560);
\node[font=\fontsize{24}{28}\selectfont, fill=white, inner sep=4pt]
  at (1290,575) {$\times\,N$};

\draw[normframe] (2342,295) rectangle (2418,385);
\node[fL, align=center] at (2380,250) {Final\\[-2pt]RMSNorm};
\node[normbox] (nf) at (2380,340) {$\dfrac{\mathbf{x}}{\|\mathbf{x}\|}$};
\node[weight, minimum width=160*\dpx] (head) at (2530,340)
  {$\mathbf{R}_1^{\top}\mathbf{W}_{\mathrm{head}}$};
\node[fL] at (2530,280) {LM HEAD};
\node[font=\rmfamily\fontsize{16}{19}\selectfont] (logits) at (2530,450) {LOGITS};

\draw[main] (YR) -- (nf);
\draw[main] (nf) -- (head);
\draw[main] (head) -- (logits);

\end{tikzpicture}}
  \caption{Placement of the rotations $\bm{R}_1$ to $\bm{R}_4$ in a Transformer block. Operations outlined in black are in FP16 and those outlined in blue are in INT4. The dashed lines indicate read and write operations to the KV cache.}
 \label{fig:rotational_invariance}
\end{figure*}

\subsection{Implementation Details}
Rotations ${\bm R}_1$ and ${\bm R}_2$ are calibrated using SpectrumQuant with the Sylvester-Hadamard matrix (${\bf H}$) and $45^\circ$ butterfly matrix ($\bm B$) as the mixing component (that is, ${\bm Q}={\bm {UH}} $ and $\bm{UB}$ as calibrated matrices). ${\bm R}_3$ and ${\bm R}_4$ are left online as Hadamard matrices, as in the case of the baselines. Since $\bm R_1$ acts as a single global rotation on the residual stream, $\bm U$ is calibrated once per model from activations pooled across all layers. Rotations ${\bm R}_1$ and ${\bm R}_2$ are computed offline from activation covariance statistics estimated on 128 WikiText2 calibration samples using the Welford algorithm \citep{Welford01081962, Chan1983AlgorithmsFC}, which provides a numerically stable unbiased online covariance estimator. The Welford algorithm is listed in Appendix~\ref{welford}. 

\subsection{Main Results}
We evaluate SpectrumQuant on Llama-family models, including Llama-2 (7B, 13B, and 70B)  \citep{touvron2023llama}, Llama-3 (8B) and Mistral-family models, Mistral 7B v0.3 \citep{jiang2023mistral}. We report \textit{perplexity} (PPL) on WikiText2 \citep{merity2016pointer}, C4 \citep{raffel2020t5}, and PTB \citep{marcus-etal-1993-building}, and further measure zero-shot performance on eight downstream tasks: WinoGrande \citep{winograde}, SIQA \citep{siqa}, PIQA \citep{piqa}, OpenBookQA \citep{openbookqa}, LAMBADA \citep{lambada}, HellaSwag \citep{zellers2019hellaswag}, ARC-E, and ARC-C \citep{arc}.

We compare our approach with existing rotation-based quantization approaches, including Hadamard rotations in QuaRot \citep{quarot}, end-to-end learned rotations in SpinQuant \citep{liu2025spinquantllmquantizationlearned} and DartQuant \citep{shao2026dartquant}. In our implementation, we quantize weights with GPTQ \citep{frantar2023gptqaccurateposttrainingquantization} under the standard setup using 128 WikiText2 calibration samples, each of length 2048, and quantize activations with per-token asymmetric quantization. 
\begin{wrapfigure}{r}{0.6\textwidth}
    \centering
    \captionsetup{type=table}
    \caption{Comparison of calibration cost. SpectrumQuant requires less than one-tenth the memory of SpinQuant and less than one-sixth the memory of DartQuant at 70B scale, while also being the fastest to calibrate across all model sizes.}
    \label{tab:calibration_cost_spinquant}

    \begin{tabular}{ccccc}
        \toprule
        Cost & Method & 7B & 13B & 70B \\
        \midrule
        \multirow{3}{*}{GPU hour (h)}
        & SpinQuant & 0.41 & 0.84 & 46.40 \\
        & DartQuant & 0.32 & 0.57 & 3.74 \\
        \rowcolor{tablegray}
        & SpectrumQuant & \textbf{0.26} & \textbf{0.35} & \textbf{2.25} \\
        \midrule
        \multirow{3}{*}{Memory (GiB)}
        & SpinQuant & 18.23 & 32.21 & 218.72 \\
        & DartQuant & 15.79 & 27.80 & 128.21 \\
        \rowcolor{tablegray}
        & SpectrumQuant & \textbf{4.51} & \textbf{7.05} & \textbf{22.42} \\
        \bottomrule
    \end{tabular}
\end{wrapfigure}

A comparison of rotation matrix computation time and memory cost of SpinQuant and SpectrumQuant on an Nvidia A6000 GPU server is shown in Table ~\ref{tab:calibration_cost_spinquant}. SpectrumQuant is a simple calibration approach, which has a significant reduction in compute cost across various models. For the 70B model, after pre-saving activations, SpectrumQuant completes calibration in 43 minutes on a single GPU and can also calibrate on CPU in 51 minutes, achieving a speedup of 21$\times$ in calibration and 9.8$\times$ in memory usage compared to SpinQuant. While DartQuant calibration time is relatively low, with only 10 epochs, calibration memory requirement is significantly higher. All competing baseline models are reproduced from their publicly available codebases.

\begin{table}[h]
  \caption{Comparison of average PPL across three datasets and average accuracy on eight
  zero-shot common-sense reasoning tasks (zero-shot$^8$), over 5 seeds, under W4-A4-KV4
  quantization. SpectrumQuant matches or outperforms SpinQuant on all five models and shows no
  statistically significant difference from DartQuant on four out of five models, without requiring iterative optimization. Detailed results are given in
  Appendix~\ref{app: main-results}.}
  \label{main-results-table}
  \centering
  \setlength{\tabcolsep}{3pt}
  \renewcommand{\arraystretch}{1.15}
  \newcommand{\s}[1]{{\scriptsize$\pm$#1}}
  \resizebox{\textwidth}{!}{%
  \begin{tabular}{l cccccccccc}
    \toprule
    \multirow{2}{*}{Method}
    & \multicolumn{2}{c}{Llama-2 7B} & \multicolumn{2}{c}{Llama-2 13B}
    & \multicolumn{2}{c}{Llama-3 8B} & \multicolumn{2}{c}{Mistral 7B}
    & \multicolumn{2}{c}{Llama-2 70B} \\
    \cmidrule(lr){2-3} \cmidrule(lr){4-5} \cmidrule(lr){6-7} \cmidrule(lr){8-9} \cmidrule(lr){10-11}
    & PPL$\downarrow$ & zero-shot$^8$$\uparrow$ & PPL$\downarrow$ & zero-shot$^8$$\uparrow$
    & PPL$\downarrow$ & zero-shot$^8$$\uparrow$ & PPL$\downarrow$ & zero-shot$^8$$\uparrow$
    & PPL$\downarrow$ & zero-shot$^8$$\uparrow$ \\
    \midrule
    FP16
    & 16.88 & 63.66 & 20.85 & 65.99 & 8.92 & 66.54 & 17.04 & 67.74 & 11.09 & 70.07 \\
    \midrule
    QuaRot
    & 25.57\s{1.22} & 59.73\s{0.36} & 24.64\s{0.45} & 63.34\s{0.38}
    & 12.10\s{0.13} & 59.70\s{0.61} & 20.11\s{0.42} & 64.37\s{0.21}
    & \textbf{11.68}\s{0.14} & \textbf{68.66}\s{0.36} \\
    SpinQuant
    & 22.92\s{0.67} & 59.23\s{0.54} & 24.93\s{0.88} & 63.19\s{0.50}
    & 11.67\s{0.09} & 61.36\s{0.44} & 19.50\s{0.26} & \textbf{65.18}\s{0.18}
    & 12.26\s{0.14} & \textbf{68.58}\s{0.50} \\
    DartQuant
    & \textbf{20.46}\s{0.20} & \textbf{60.99}\s{0.43} & \textbf{24.47}\s{0.60} & \textbf{64.30}\s{0.14}
    & \textbf{10.95}\s{0.05} & \textbf{62.74}\s{0.35} & \textbf{18.73}\s{0.30} & \textbf{65.34}\s{0.28}
    & 12.00\s{0.08} & \textbf{68.87}\s{0.42} \\
    \rowcolor{tablegray} SpectrumQuant ($\bm{Q}=\bm{UH}$)
    & 21.50\s{0.46} & 60.06\s{0.32} & \textbf{24.53}\s{0.56} & \textbf{63.71}\s{0.58}
    & \textbf{10.99}\s{0.12} & \textbf{62.57}\s{0.50} & \textbf{18.77}\s{0.17} & \textbf{65.45}\s{0.34}
    & 12.00\s{0.11} & \textbf{68.94}\s{0.54} \\
    \rowcolor{tablegray} SpectrumQuant ($\bm{Q}=\bm{UB}$)
    & \textbf{21.44}\s{1.03} & 60.29\s{0.26} & \textbf{23.89}\s{0.54} & \textbf{64.12}\s{0.65}
    & \textbf{10.95}\s{0.08} & \textbf{62.84}\s{0.20} & \textbf{18.98}\s{0.18} & \textbf{65.34}\s{0.35}
    & \textbf{11.74}\s{0.11} & \textbf{68.62}\s{0.37} \\
    \bottomrule
  \end{tabular}%
  }
\end{table}

Table~\ref{main-results-table} compares our 4-bit weight, activation and KV cache approach against FloatingPoint, QuaRot, SpinQuant and DartQuant on three WikiText-style PPL evaluations and the average accuracy over eight zero-shot common-sense reasoning tasks (zero-shot$^8$). Across the LLM architectures, our approach matches or outperforms SpinQuant on all five models, with the clearest gains on Llama-2 7B and Llama-3 8B and on par with DartQuant on four out of five models, with DartQuant ahead on Llama-2 7B. In our SpectrumQuant implementation, we have demonstrated that the orthogonal transformations using the mixing components that are more general than the structured Sylvester-Hadamard matrices, as described in Section~\ref{sec:all_sols}, perform almost equally. This further strengthens the theoretical contribution of the present work. In fact, our theoretical results also explain why rotation methods involving only butterfly or only Hadamard matrices are effective in LLM post-training quantization.

Compared to QuaRot, SpectrumQuant reduces PPL by 16.2\% on Llama-2 7B, 3\% on Llama-2 13B, 9.5\% on Llama-3 8B and 5.6\% on Mistral-v0.3 7B on average. Our method improves zero-shot accuracy by 0.56 percentage points on Llama-2 7B, 0.78 percentage points on Llama-2 13B, 3.14 percentage points on Llama-3 8B, 1.08 percentage points on Mistral-v0.3 7B, and 0.28 percentage points on Llama-2 70B on average. Additionally, in Section \ref{sec:ablation-Hadamard}, we have shown that our closed-form COAT is consistently better than the Hadamard rotation in terms of the mean-squared error (MSE) due to quantization.

\subsection{Ablation Studies}
\label{sec:ablation}
To better understand the individual contributions of each component of COAT in our proposed method SpectrumQuant, we conducted some ablation studies as described below.

\subsubsection{Calibration with different datasets}
To study sensitivity of the calibration-data, we adopt the SpectrumQuant approach using WikiText2, PTB, and C4 samples and evaluate the resulting quantized models. As shown in Table~\ref{calibration-dataset}, the results are broadly consistent across calibration datasets for both Llama-2 7B and Mistral 7B v0.3, suggesting that the proposed approach is robust to the choice of calibration data.

\begin{table}[t]
\centering
\caption{Comparison of PPL of SpectrumQuant (with ${\bm Q}=\bm{UH}$) across various calibration datasets for 5 seeds. SpectrumQuant shows no statistically significant difference across calibration datasets for Mistral 7B v0.3, and remains competitive across calibration datasets for Llama-2 7B, suggesting the approach is generally robust to calibration data selection.} \label{calibration-dataset}
\begin{tabular}{cccccc} 
\toprule
Model & Dataset &  WikiText2 & PTB & C4  \\
\midrule
\multirow{3}{*}{Llama-2 7B}
 & WikiText2 & $6.28_{\pm 0.06}$  & $49.86_{\pm 1.31}$ & $8.37_{\pm 0.06}$   \\
 & PTB       & $6.23_{\pm 0.029}$ & $44.95_{\pm 1.42}$ & $8.392_{\pm 0.066}$ \\
 & C4        & $6.28_{\pm 0.041}$ & $46.90_{\pm 0.65}$ & $8.438_{\pm 0.060}$ \\
\midrule
\multirow{3}{*}{Mistral 7B v0.3}
 & WikiText2 & $5.67_{\pm 0.01}$   & $41.64_{\pm 0.50}$ & $9.00_{\pm 0.00}$   \\
 & PTB       & $5.74_{\pm 0.014}$ & $40.51_{\pm 0.63}$ & $9.016_{\pm 0.027}$ \\
 & C4        & $5.70_{\pm 0.010}$ & $42.56_{\pm 1.21}$ & $8.983_{\pm 0.015}$ \\
\bottomrule
\end{tabular}
\end{table}

\subsubsection{Calibration with different sample sizes}
To study sensitivity to the sample size of the calibration data, we tested different numbers of samples from the WikiText2 Dataset for the Llama-2 7B model and analyzed the PPL scores. As shown in Table~\ref{tab:sample-size-sensitivity}, the results are broadly consistent across the choice of number of calibration samples, suggesting that SpectrumQuant is robust to sample size. By default, in the rest of the paper, we have used the sample size of 128, consistent with other calibrated approaches like SpinQuant and DartQuant.
\begin{table}[H]
\centering
\caption{PPL comparison of SpectrumQuant (with ${\bm Q}=\bm{UH}$) across sample sizes for Llama-2 7B. SpectrumQuant shows no statistically significant difference in performance from 64 to 256 calibration samples, suggesting that the dominant covariance structure is reliably captured even at small sample counts.}
\label{tab:sample-size-sensitivity}
\begin{tabular}{cccc c}
\toprule
No. of samples &  WikiText2 & PTB & C4 & Avg. PPL$\downarrow$ \\
\midrule
64  & 6.169$_{\pm0.033}$ & 49.07$_{\pm2.24}$ & 8.393$_{\pm0.065}$ & 21.21$_{\pm0.76}$ \\
128 & 6.28$_{\pm0.06}$   & 49.86$_{\pm1.31}$ & 8.37$_{\pm0.06}$   & 21.50$_{\pm0.46}$ \\
256 & 6.188$_{\pm0.038}$ & 50.22$_{\pm2.12}$ & 8.429$_{\pm0.023}$ & 21.61$_{\pm0.72}$ \\
\bottomrule
\end{tabular}
\end{table}

\subsubsection{Comparison with the Hadamard transform}
\label{sec:ablation-Hadamard}
The quantization MSE across different latent dimensions of the LLM architecture is shown in Figure \ref{fig:layerwise_quant_loss_main} corresponding to layers 10 and 25 of the Llama-2 7B model. We observe that the projections ($q/k/v/o, gate/up$) in the case of SpectrumQuant have lower MSE than those in the Hadamard-rotation-based approach. For $down$ projection, both competing approaches use the same $\bm R_4$ matrix. On the average, we have 2\% less MSE than the Hadamard-rotation-based approach across almost all layers.  Appendix~\ref{appendix:layerwise-quant-error} shows the MSE for several layers. The improvement is observed in both intermediate and deeper layers, suggesting that SpectrumQuant remains stable across network depths and projection types. 
\begin{figure}[H]
  \centering
  \includegraphics[trim={0cm 0cm 0cm 0cm}, clip, width=\linewidth]{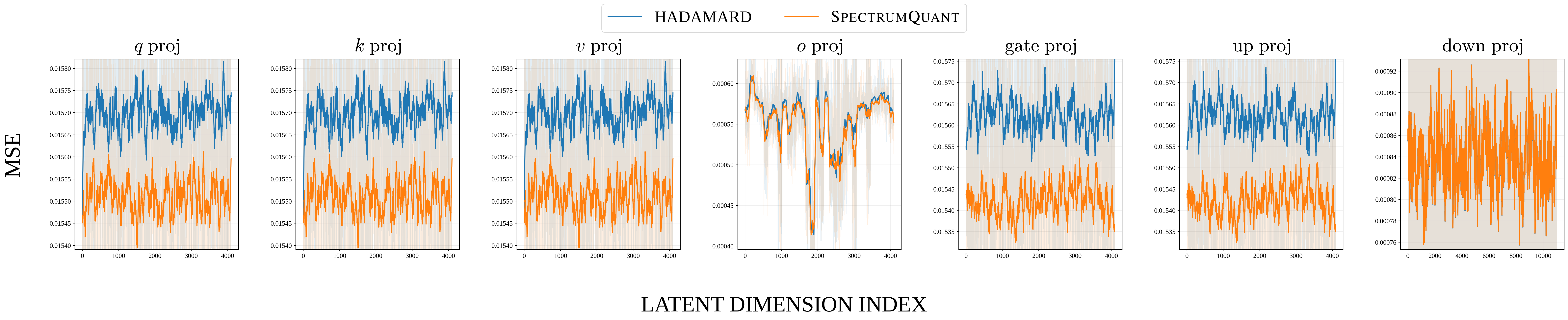}
  \caption{Comparison of sorted per-channel quantization MSE across hidden dimensions for different projection activations in Llama-2 7B at layers 10 and 25.}
  \label{fig:layerwise_quant_loss_main}
\end{figure}

\subsubsection{Butterfly as a replacement for the Sylvester-Hadamard matrix}
\label{sec:ablation-butterfly}
Throughout the existing literature, the Sylvester-Hadamard matrix has been considered the default choice for ${\bm R}_3, {\bm R}_4$. Following the analysis in Section~\ref{sec:all_sols}, we hypothesize that any $45^\circ$ butterfly matrix, as in \ref{eq:butterfly_matrix_i}, should also serve the purpose. A general butterfly transformation requires $\mathcal{O}(d\log d)$ operations, which is comparable to that of the Sylvester-Hadamard matrix.
Since the entries of the $45^\circ$ butterfly matrix or Sylvester-Hadamard matrix are $\pm 1$, the transformation requires only addition/subtraction operations,  significantly reducing the computational cost. To test the hypothesis, we replace ${\bm R}_3, {\bm R}_4$ in different LLM architectures with random butterfly matrices instead of random Sylvester-Hadamard matrices. Table~\ref{tab:butterfly-replacement} gives the PPL and zero-shot$^8$ for different models with $\bm{Q} = \bm{UB} $ as the calibrated ${\bm R}_1,{\bm R}_2$ and ${\bm B}$ in place of ${\bm R}_3$ and ${\bm R}_4$. Comparing the results in Table~\ref{main-results-table} with Table~\ref{tab:butterfly-replacement}, we observe that the performance of butterfly matrices is comparable to that of Sylvester-Hadamard matrices. These findings establish that the $45^\circ$ butterfly matrix can be used in place of the Sylvester-Hadamard matrix for $\bm{R_3}$ and $\bm{R_4}$ without compromising performance.

\begin{table}[t]
\centering
\caption{PPL and average accuracy for zero-shot common-sense reasoning task with calibrated $\bm R_1,\bm R_2$ as $\bm{UB}$ and ${\bm R_3}, {\bm R_4}$ as random butterfly matrices. Replacing the standard Hadamard matrices in ${\bm R_3}, {\bm R_4}$  with random $45^\circ$ butterfly matrices shows no statistically significant difference in results.}
\label{tab:butterfly-replacement}
\begin{tabular}{lcccc c}
\toprule
Model &  WikiText2  & PTB  & C4 & Avg. PPL $\downarrow$ & zero-shot$^8$ $\uparrow$ \\
\midrule
Llama-2 7B  & 6.21$_{\pm0.04}$ & 49.93$_{\pm1.00}$ & 8.49$_{\pm0.10}$  & 21.54$_{\pm0.38}$ & 60.26$_{\pm0.46}$ \\
Llama-3 8B  & 7.74$_{\pm0.06}$ & 13.34$_{\pm0.07}$ & 12.57$_{\pm0.09}$  & 11.22$_{\pm0.07}$ & 62.85$_{\pm0.74}$ \\
Mistral 7B v0.3 & 5.66$_{\pm0.02}$ & 43.03$_{\pm1.20}$ & 9.01$_{\pm0.01}$  & 19.23$_{\pm0.41}$ & 64.56$_{\pm0.30}$ \\
\bottomrule
\end{tabular}
\end{table}

\subsection{Inference Cost}
Due to commutativity of the RMSNorm operation and an orthogonal transform, rotations can be absorbed into the weights in the LLM architecture, essentially requiring the same inference cost as that of the no-transform baseline; that is, without applying rotation transforms prior to quantization. This is also demonstrated by the results shown in Table~\ref{tab:inference-cost}, obtained by running all the quantized models on a MacBook Air with the following specifications: M4 chip, 16 GB RAM, 10-core CPU, 10-core GPU, and 16-core neural engine.

\begin{table}[H]
\centering
\caption{Real-time inference cost of Llama-3 8B on MacBook Air M4. The rotation-based approaches including SpectrumQuant show no significant difference in inference throughput compared to the no-transform baseline (Quant), confirming that the rotation matrices are fully absorbed into the weights with no additional inference cost. Quant corresponds to quantization without applying a rotation transform.}
\label{tab:inference-cost}
\begin{tabular}{lcccc}
\toprule
 & Quant & SpinQuant & DartQuant & SpectrumQuant \\
\midrule
Decode Tokens/sec $\uparrow$ & 18.330 & 18.337 & 18.246 & 18.397 \\
Decode ms/Token $\downarrow$ & 54.6 & 54.5 & 54.8 & 54.4 \\
\bottomrule
\end{tabular}
\end{table}

\section{Conclusions}
We presented SpectrumQuant, a novel and data-aware framework for post-training quantization of large language models. SpectrumQuant relies on a COvariance-Aware orthogonal Transform (COAT) derived in \textit{closed-form} that estimates activation covariance from calibration data. COAT naturally emerges from a covariance-maximizing objective over the Stiefel manifold. Consequently, we bridged a gap in the LLM quantization literature where rotation selection for activation outlier mitigation is either data-agnostic, such as Hadamard rotations, or based on calibration-time optimization of learned rotations. Our analysis reveals a family of optimal transforms for faithful LLM quantization, and also justifies why the Hadamard and butterfly rotations employed in the recent literature on LLM quantization are effective. 
Empirical results across diverse LLM architectures showed consistent gains over data-agnostic approaches while being competitive with end-to-end learned baselines.

\bibliography{references}
\bibliographystyle{tmlr} 

\newpage
\appendix

\section*{Appendix}
\section{Welford Algorithm for Covariance Estimation}
The orthogonal matrices for SpectrumQuant are computed offline from activation covariance statistics estimated on the same 128 WikiText2 calibration samples using the Welford algorithm \citep{Welford01081962} \citep{Chan1983AlgorithmsFC} given below: 

\label{welford}
\begin{algorithm}[H]
\caption{Welford algorithm for covariance estimation from calibration activations}
\label{alg:welford_cov}
\begin{algorithmic}[1]
\Require Calibration samples $\{\bm x_t\}_{t=1}^{N}$, where $\bm x_t \in \mathbb{R}^{d}$
\Ensure Mean $\bm{\mu} \in \mathbb{R}^{d}$ and covariance estimate
$\bm{\Sigma} \in \mathbb{R}^{d\times d}$

\State $\bm{\mu} \gets \bm{0}$
\State $\bm{\Sigma} \gets \bm{0}_{d\times d}$

\For{$t = 1$ to $N$}
    \State $\bm{\delta} \gets \bm x_t - \bm{\mu}$
    \State $\bm{\mu} \gets \bm{\mu} + \bm{\delta}/t$
    \State $\bm{\delta}' \gets \bm x_t - \bm{\mu}$
    \State $\bm{\Sigma} \gets \bm{\Sigma} + \bm{\delta}{\bm{\delta}'}^\top$
\EndFor

\State $\bm{\Sigma} \gets \bm{\Sigma}/(N-1)$
\State \Return $\bm{\mu}, \bm{\Sigma}$
\end{algorithmic}
\end{algorithm}

\section{Detailed Results Pertaining to Table~\ref{main-results-table}}
\label{app: main-results}
We evaluate SpectrumQuant versus the baselines on Llama-family models, including Llama-2 (7B, 13B, and 70B)  \citep{touvron2023llama}, Llama-3 (8B) and Mistral-family models, Mistral 7B v0.3 \citep{jiang2023mistral}. 

This section reports perplexity on each of the three evaluation corpora and accuracy on each of the eight common-sense benchmarks, for both the Llama and Mistral family models. All the experiments suggest that SpectrumQuant shows statistically similar performance across different settings of the mixing component ${\bm V}$ in the underlying transformation ${\bm Q}={\bm U}{\bm V}$. This shows that the empirical outcomes are consistent with our theoretical claims about SpectrumQuant. Furthermore, it is observed that  SpectrumQuant outperforms both the baselines QuaRot and SpinQuant across different metrics. Additionally, SpectrumQuant outcomes show no significant difference from DartQuant, despite being constructed entirely without iterative optimization and at a fraction of the calibration memory cost. All the tables report the average PPL or zero-shot accuracy along with the standard deviation in the subscript.


\begin{table}[H]
\caption{Perplexity of each method on WikiText2, PTB and C4 for Llama-2 7B. The two SpectrumQuant variants (with ${\bm Q}=\bm{UH}$ and ${\bm Q}=\bm{UB}$) are statistically indistinguishable from each other on every dataset and lie between QuaRot and DartQuant, indicating that the choice of mixing component does not affect performance.}
\centering
\begin{tabular}{clcccc}
\toprule
\multicolumn{1}{c}{Bits} & \multicolumn{1}{c}{\multirow{2}{*}{Method}} & \multirow{2}{*}{Wiki} & \multirow{2}{*}{PTB} & \multirow{2}{*}{C4} & \multirow{2}{*}{Avg $\downarrow$} \\
\multicolumn{1}{c}{{(W-A-KV)}} & & & & & \\ \midrule
16 & Full Precision & 5.47 & 37.91 & 7.26 & 16.88 \\ \midrule
\multirow{4}{*}{4-4-4}
 & QuaRot & $6.21_{\pm 0.01}$ & $62.05_{\pm 3.68}$ & $8.45_{\pm 0.01}$ & $25.57_{\pm 1.22}$ \\
 & SpinQuant & $6.49_{\pm 0.04}$ & $53.50_{\pm 1.92}$ & $8.77_{\pm 0.10}$ & $22.92_{\pm 0.67}$ \\
 & DartQuant & $6.10_{\pm 0.02}$ & $47.17_{\pm 0.59}$ & $8.11_{\pm 0.02}$ & $20.46_{\pm 0.20}$ \\
 \rowcolor{tablegray} & SpectrumQuant ($\bm{Q}=\bm{UH}$) & $6.28_{\pm 0.06}$ & $49.86_{\pm 1.31}$ & $8.37_{\pm 0.06}$ & $21.50_{\pm 0.46}$ \\
 \rowcolor{tablegray} & SpectrumQuant ($\bm{Q}=\bm{UB}$) & $6.18_{\pm 0.07}$ & $49.69_{\pm 2.96}$ & $8.44_{\pm 0.07}$ & $21.44_{\pm 1.03}$ \\ \bottomrule
\end{tabular}
\end{table}

\begin{table}[H]
\centering
\caption{Accuracy on each of the eight common-sense benchmarks for Llama-2 7B, evaluated without in-context examples. In average accuracy, both SpectrumQuant variants (${\bm Q}=\bm{UH}$ and ${\bm Q}=\bm{UB}$) are comparable to DartQuant and ahead of SpinQuant.}
\resizebox{\textwidth}{!}{
\begin{tabular}{clccccccccc}
\toprule
\multicolumn{1}{c}{Bits} & \multicolumn{1}{c}{\multirow{2}{*}{Method}} & \multirow{2}{*}{WG} & \multirow{2}{*}{SIQA} & \multirow{2}{*}{PIQA} & \multirow{2}{*}{OBQA} & \multirow{2}{*}{LAMB} & \multirow{2}{*}{HS} & \multirow{2}{*}{ARC-E} & \multirow{2}{*}{ARC-C} & \multirow{2}{*}{Avg $\uparrow$} \\
\multicolumn{1}{c}{(W-A-KV)} & & & & & & & & & & \\ \midrule
16 & Full Precision & 69.06 & 46.16 & 79.05 & 44.20 & 73.90 & 76.02 & 74.54 & 46.33 & 63.66 \\ \midrule
\multirow{4}{*}{4-4-4}
 & QuaRot & $66.17_{\pm 0.28}$ & $44.66_{\pm 0.62}$ & $76.61_{\pm 0.44}$ & $39.12_{\pm 0.98}$ & $69.37_{\pm 0.60}$ & $72.61_{\pm 0.37}$ & $68.01_{\pm 1.17}$ & $41.30_{\pm 1.32}$ & $59.73_{\pm 0.36}$ \\
 & SpinQuant & $65.35_{\pm 1.27}$ & $43.64_{\pm 0.33}$ & $76.36_{\pm 0.71}$ & $39.52_{\pm 2.05}$ & $68.27_{\pm 0.85}$ & $71.45_{\pm 0.90}$ & $68.25_{\pm 1.05}$ & $40.99_{\pm 1.35}$ & $59.23_{\pm 0.54}$ \\
 & DartQuant & $67.44_{\pm 0.46}$ & $45.12_{\pm 0.74}$ & $77.20_{\pm 0.43}$ & $41.04_{\pm 1.69}$ & $71.17_{\pm 0.36}$ & $72.98_{\pm 0.25}$ & $70.09_{\pm 1.11}$ & $42.92_{\pm 0.99}$ & $60.99_{\pm 0.43}$ \\
 \rowcolor{tablegray} & SpectrumQuant ($\bm{Q}=\bm{UH}$) & $66.99_{\pm 0.71}$ & $44.11_{\pm 1.07}$ & $77.10_{\pm 0.64}$ & $39.24_{\pm 1.00}$ & $70.46_{\pm 0.60}$ & $72.10_{\pm 0.43}$ & $68.76_{\pm 1.34}$ & $41.67_{\pm 1.30}$ & $60.06_{\pm 0.32}$ \\
 \rowcolor{tablegray} & SpectrumQuant ($\bm{Q}=\bm{UB}$) & $66.53_{\pm 0.65}$ & $43.62_{\pm 0.62}$ & $77.13_{\pm 0.41}$ & $40.32_{\pm 0.27}$ & $70.24_{\pm 0.71}$ & $72.38_{\pm 0.25}$ & $69.64_{\pm 1.07}$ & $42.46_{\pm 0.79}$ & $60.29_{\pm 0.26}$ \\ \bottomrule
\end{tabular}
}
\end{table}


\begin{table}[H]
\caption{Perplexity of each method on WikiText2, PTB and C4 for Llama-2 13B. SpectrumQuant
(${\bm Q}=\bm{UB}$) attains the lowest average perplexity of all approaches, and both variants
(${\bm Q}=\bm{UH}$ and ${\bm Q}=\bm{UB}$) improve on QuaRot and SpinQuant.}
\centering
\begin{tabular}{clcccc}
\toprule
\multicolumn{1}{c}{{Bits}} & \multicolumn{1}{c}{\multirow{2}{*}{{Method}}} & \multirow{2}{*}{{Wiki}} & \multirow{2}{*}{{PTB}} & \multirow{2}{*}{{C4}} & \multirow{2}{*}{{Avg $\downarrow$}} \\
\multicolumn{1}{c}{(W-A-KV)} & & & & & \\ \midrule
16 & Full Precision & 4.88 & 50.94 & 6.73 & 20.85 \\ \midrule
\multirow{4}{*}{4-4-4}
 & QuaRot & $5.49_{\pm 0.01}$ & $60.79_{\pm 1.33}$ & $7.64_{\pm 0.01}$ & $24.64_{\pm 0.45}$ \\
 & SpinQuant & $5.61_{\pm 0.02}$ & $61.39_{\pm 2.62}$ & $7.78_{\pm 0.02}$ & $24.93_{\pm 0.88}$ \\
 & DartQuant & $5.40_{\pm 0.01}$ & $60.64_{\pm 1.79}$ & $7.37_{\pm 0.01}$ & $24.47_{\pm 0.60}$ \\
 \rowcolor{tablegray} & SpectrumQuant ($\bm{Q}=\bm{UH}$) & $5.54_{\pm 0.04}$ & $60.48_{\pm 1.70}$ & $7.58_{\pm 0.05}$ & $24.53_{\pm 0.56}$ \\
 \rowcolor{tablegray} & SpectrumQuant ($\bm{Q}=\bm{UB}$) & $5.44_{\pm 0.01}$ & $58.62_{\pm 1.63}$ & $7.60_{\pm 0.03}$ & $23.89_{\pm 0.54}$ \\ \bottomrule
\end{tabular}
\end{table}

\begin{table}[H]
\centering
\caption{Accuracy on each of the eight common-sense benchmarks for Llama-2 13B, evaluated without in-context examples. SpectrumQuant (with ${\bm Q}=\bm{UB}$) matches DartQuant in average accuracy within statistical error, without any iterative optimization and at a small fraction of its calibration memory.}
\resizebox{\textwidth}{!}{
\begin{tabular}{clccccccccc}
\toprule
\multicolumn{1}{c}{Bits} & \multicolumn{1}{c}{\multirow{2}{*}{Method}} & \multirow{2}{*}{WG} & \multirow{2}{*}{SIQA} & \multirow{2}{*}{PIQA} & \multirow{2}{*}{OBQA} & \multirow{2}{*}{LAMB} & \multirow{2}{*}{HS} & \multirow{2}{*}{ARC-E} & \multirow{2}{*}{ARC-C} & \multirow{2}{*}{Avg $\uparrow$} \\
\multicolumn{1}{c}{(W-A-KV)} & & & & & & & & & & \\ \midrule
16 & Full Precision & 72.14 & 47.39 & 80.52 & 45.20 & 76.73 & 79.38 & 77.44 & 49.15 & 65.99 \\ \midrule
\multirow{4}{*}{4-4-4}
 & QuaRot & $70.02_{\pm 0.42}$ & $45.77_{\pm 0.52}$ & $78.71_{\pm 0.90}$ & $43.92_{\pm 1.19}$ & $72.97_{\pm 0.60}$ & $76.29_{\pm 0.60}$ & $73.28_{\pm 0.95}$ & $45.74_{\pm 0.81}$ & $63.34_{\pm 0.38}$ \\
 & SpinQuant & $69.36_{\pm 0.93}$ & $45.65_{\pm 0.94}$ & $78.70_{\pm 0.68}$ & $43.92_{\pm 0.66}$ & $72.98_{\pm 1.25}$ & $76.05_{\pm 0.38}$ & $72.97_{\pm 1.24}$ & $45.90_{\pm 0.45}$ & $63.19_{\pm 0.50}$ \\
 & DartQuant & $70.40_{\pm 1.09}$ & $45.92_{\pm 0.21}$ & $79.05_{\pm 0.56}$ & $44.48_{\pm 1.58}$ & $75.28_{\pm 0.59}$ & $77.09_{\pm 0.24}$ & $74.66_{\pm 1.23}$ & $47.48_{\pm 1.20}$ & $64.30_{\pm 0.14}$ \\
 \rowcolor{tablegray} & SpectrumQuant ($\bm{Q}=\bm{UH}$) & $70.48_{\pm 1.07}$ & $45.05_{\pm 0.54}$ & $79.09_{\pm 0.46}$ & $43.64_{\pm 1.00}$ & $74.77_{\pm 1.32}$ & $76.46_{\pm 0.49}$ & $74.16_{\pm 1.45}$ & $46.01_{\pm 1.15}$ & $63.71_{\pm 0.58}$ \\
 \rowcolor{tablegray} & SpectrumQuant ($\bm{Q}=\bm{UB}$) & $71.00_{\pm 0.90}$ & $46.19_{\pm 0.67}$ & $79.43_{\pm 0.43}$ & $44.08_{\pm 1.81}$ & $74.68_{\pm 0.32}$ & $76.32_{\pm 0.42}$ & $74.37_{\pm 1.10}$ & $46.88_{\pm 1.32}$ & $64.12_{\pm 0.65}$ \\ \bottomrule
\end{tabular}
}
\end{table}


\begin{table}[H]
\caption{Perplexity of each method on WikiText2, PTB and C4 for Llama-3 8B. Both SpectrumQuant variants (with ${\bm Q}=\bm{UH}$ and ${\bm Q}=\bm{UB}$) match DartQuant in average perplexity within statistical error and improve on QuaRot and SpinQuant.}
\centering
\begin{tabular}{clcccc}
\toprule
\multicolumn{1}{c}{{Bits}} & \multicolumn{1}{c}{\multirow{2}{*}{{Method}}} & \multirow{2}{*}{{Wiki}} & \multirow{2}{*}{{PTB}} & \multirow{2}{*}{{C4}} & \multirow{2}{*}{{Avg $\downarrow$}} \\
\multicolumn{1}{c}{(W-A-KV)} & & & & & \\ \midrule
16 & Full Precision & 6.14 & 11.18 & 9.45 & 8.92 \\ \midrule
\multirow{4}{*}{4-4-4}
 & QuaRot & $8.29_{\pm 0.09}$ & $14.48_{\pm 0.21}$ & $13.54_{\pm 0.10}$ & $12.10_{\pm 0.13}$ \\
 & SpinQuant & $8.12_{\pm 0.10}$ & $13.69_{\pm 0.09}$ & $13.19_{\pm 0.14}$ & $11.67_{\pm 0.09}$ \\
 & DartQuant & $7.52_{\pm 0.03}$ & $13.17_{\pm 0.07}$ & $12.16_{\pm 0.06}$ & $10.95_{\pm 0.05}$ \\
 \rowcolor{tablegray} & SpectrumQuant ($\bm{Q}=\bm{UH}$) & $7.66_{\pm 0.06}$ & $12.96_{\pm 0.15}$ & $12.34_{\pm 0.14}$ & $10.99_{\pm 0.12}$ \\
 \rowcolor{tablegray} & SpectrumQuant ($\bm{Q}=\bm{UB}$) & $7.58_{\pm 0.04}$ & $12.98_{\pm 0.11}$ & $12.28_{\pm 0.11}$ & $10.95_{\pm 0.08}$ \\ \bottomrule
\end{tabular}
\end{table}

\begin{table}[H]
\centering
\caption{Accuracy on each of the eight common-sense benchmarks for Llama-3 8B, evaluated without in-context examples. Both SpectrumQuant variants (with ${\bm Q}=\bm{UH}$ and ${\bm Q}=\bm{UB}$) match DartQuant in average accuracy within statistical error and exceed SpinQuant on every individual benchmark, while requiring far less calibration memory.}
\resizebox{\textwidth}{!}{
\begin{tabular}{clccccccccc}
\toprule
\multicolumn{1}{c}{Bits} & \multicolumn{1}{c}{\multirow{2}{*}{Method}} & \multirow{2}{*}{WG} & \multirow{2}{*}{SIQA} & \multirow{2}{*}{PIQA} & \multirow{2}{*}{OBQA} & \multirow{2}{*}{LAMB} & \multirow{2}{*}{HS} & \multirow{2}{*}{ARC-E} & \multirow{2}{*}{ARC-C} & \multirow{2}{*}{Avg $\uparrow$} \\
\multicolumn{1}{c}{(W-A-KV)} & & & & & & & & & & \\ \midrule
16 & Full Precision & 73.32 & 47.19 & 80.96 & 45.00 & 75.63 & 79.14 & 77.82 & 53.24 & 66.54 \\ \midrule
\multirow{4}{*}{4-4-4}
 & QuaRot & $67.01_{\pm 0.90}$ & $43.88_{\pm 0.76}$ & $74.94_{\pm 0.49}$ & $40.12_{\pm 0.59}$ & $66.11_{\pm 1.16}$ & $73.03_{\pm 0.48}$ & $68.86_{\pm 2.69}$ & $43.60_{\pm 1.72}$ & $59.70_{\pm 0.61}$ \\
 & SpinQuant & $68.59_{\pm 1.14}$ & $44.65_{\pm 0.89}$ & $75.88_{\pm 1.38}$ & $41.04_{\pm 0.71}$ & $68.57_{\pm 1.83}$ & $73.59_{\pm 0.72}$ & $72.41_{\pm 1.18}$ & $46.18_{\pm 1.19}$ & $61.36_{\pm 0.44}$ \\
 & DartQuant & $69.41_{\pm 0.98}$ & $45.17_{\pm 0.57}$ & $76.97_{\pm 0.42}$ & $42.84_{\pm 1.27}$ & $70.00_{\pm 0.62}$ & $75.38_{\pm 0.42}$ & $73.93_{\pm 0.58}$ & $48.21_{\pm 0.34}$ & $62.74_{\pm 0.35}$ \\
 \rowcolor{tablegray} & SpectrumQuant ($\bm{Q}=\bm{UH}$) & $70.10_{\pm 0.84}$ & $45.54_{\pm 1.04}$ & $76.22_{\pm 1.09}$ & $41.68_{\pm 0.72}$ & $71.21_{\pm 1.31}$ & $75.26_{\pm 0.47}$ & $73.07_{\pm 1.22}$ & $47.44_{\pm 0.89}$ & $62.57_{\pm 0.50}$ \\
 \rowcolor{tablegray} & SpectrumQuant ($\bm{Q}=\bm{UB}$) & $69.57_{\pm 0.55}$ & $45.30_{\pm 0.63}$ & $77.45_{\pm 0.63}$ & $42.36_{\pm 1.48}$ & $71.34_{\pm 1.09}$ & $75.12_{\pm 0.27}$ & $73.50_{\pm 1.13}$ & $48.07_{\pm 0.84}$ & $62.84_{\pm 0.20}$ \\ \bottomrule
\end{tabular}
}
\end{table}


\begin{table}[H]
\caption{Perplexity of each method on WikiText2, PTB and C4 for Mistral 7B v0.3. SpectrumQuant (with ${\bm Q}=\bm{UH}$) matches DartQuant within statistical error on all three datasets, showing that the closed-form alignment carries over to a model family beyond Llama.}
\centering
\begin{tabular}{clcccc}
\toprule
\multicolumn{1}{c}{{Bits}} & \multicolumn{1}{c}{\multirow{2}{*}{{Method}}} & \multirow{2}{*}{{Wiki}} & \multirow{2}{*}{{PTB}} & \multirow{2}{*}{{C4}} & \multirow{2}{*}{{Avg $\downarrow$}} \\
\multicolumn{1}{c}{(W-A-KV)} & & & & & \\ \midrule
16 & Full Precision & 5.32 & 37.34 & 8.48 & 17.04 \\ \midrule
\multirow{4}{*}{4-4-4}
 & QuaRot & $5.82_{\pm 0.02}$ & $45.27_{\pm 1.49}$ & $9.25_{\pm 0.01}$ & $20.11_{\pm 0.42}$ \\
 & SpinQuant & $5.76_{\pm 0.01}$ & $43.53_{\pm 0.77}$ & $9.20_{\pm 0.02}$ & $19.50_{\pm 0.26}$ \\
 & DartQuant & $5.67_{\pm 0.01}$ & $41.50_{\pm 0.90}$ & $9.02_{\pm 0.00}$ & $18.73_{\pm 0.30}$ \\
 \rowcolor{tablegray} & SpectrumQuant ($\bm{Q}=\bm{UH}$) & $5.67_{\pm 0.01}$ & $41.64_{\pm 0.50}$ & $9.00_{\pm 0.00}$ & $18.77_{\pm 0.17}$ \\
 \rowcolor{tablegray} & SpectrumQuant ($\bm{Q}=\bm{UB}$) & $5.67_{\pm 0.01}$ & $42.27_{\pm 0.53}$ & $9.01_{\pm 0.01}$ & $18.98_{\pm 0.18}$ \\ \bottomrule
\end{tabular}
\end{table}

\begin{table}[H]
\centering
\caption{Accuracy on each of the eight common-sense benchmarks for Mistral 7B v0.3, evaluated without in-context examples. SpectrumQuant (with ${\bm Q}=\bm{UH}$) achieves the highest average accuracy of all approaches, statistically tied with DartQuant, further supporting generalization across architectures.}
\resizebox{\textwidth}{!}{
\begin{tabular}{clccccccccc}
\toprule
\multicolumn{1}{c}{Bits} & \multicolumn{1}{c}{\multirow{2}{*}{Method}} & \multirow{2}{*}{WG} & \multirow{2}{*}{SIQA} & \multirow{2}{*}{PIQA} & \multirow{2}{*}{OBQA} & \multirow{2}{*}{LAMB} & \multirow{2}{*}{HS} & \multirow{2}{*}{ARC-E} & \multirow{2}{*}{ARC-C} & \multirow{2}{*}{Avg $\uparrow$} \\
\multicolumn{1}{c}{(W-A-KV)} & & & & & & & & & & \\ \midrule
16 & Full Precision & 74.74 & 47.65 & 82.05 & 47.60 & 75.00 & 80.67 & 80.01 & 54.18 & 67.74 \\ \midrule
\multirow{4}{*}{4-4-4}
 & QuaRot & $70.56_{\pm 0.56}$ & $45.24_{\pm 0.59}$ & $79.75_{\pm 0.50}$ & $44.20_{\pm 0.76}$ & $71.96_{\pm 0.32}$ & $77.98_{\pm 0.34}$ & $76.03_{\pm 0.40}$ & $49.23_{\pm 0.92}$ & $64.37_{\pm 0.21}$ \\
 & SpinQuant & $72.04_{\pm 0.73}$ & $45.66_{\pm 1.02}$ & $80.13_{\pm 0.45}$ & $44.68_{\pm 1.12}$ & $73.60_{\pm 0.36}$ & $78.31_{\pm 0.42}$ & $76.79_{\pm 0.67}$ & $50.21_{\pm 1.25}$ & $65.18_{\pm 0.18}$ \\
 & DartQuant & $71.63_{\pm 0.67}$ & $45.64_{\pm 0.50}$ & $80.60_{\pm 0.70}$ & $45.88_{\pm 1.14}$ & $72.90_{\pm 0.28}$ & $78.76_{\pm 0.28}$ & $77.08_{\pm 0.36}$ & $50.26_{\pm 0.89}$ & $65.34_{\pm 0.28}$ \\
 \rowcolor{tablegray} & SpectrumQuant ($\bm{Q}=\bm{UH}$) & $72.63_{\pm 0.88}$ & $46.03_{\pm 0.20}$ & $80.91_{\pm 0.37}$ & $44.36_{\pm 1.28}$ & $73.32_{\pm 1.11}$ & $78.68_{\pm 0.44}$ & $77.02_{\pm 0.54}$ & $50.66_{\pm 0.80}$ & $65.45_{\pm 0.34}$ \\
 \rowcolor{tablegray} & SpectrumQuant ($\bm{Q}=\bm{UB}$) & $72.71_{\pm 0.72}$ & $46.19_{\pm 0.93}$ & $80.64_{\pm 0.40}$ & $44.28_{\pm 0.67}$ & $73.23_{\pm 0.18}$ & $78.35_{\pm 0.43}$ & $77.17_{\pm 1.08}$ & $50.17_{\pm 0.56}$ & $65.34_{\pm 0.35}$ \\ \bottomrule
\end{tabular}
}
\end{table}


\begin{table}[H]
\caption{Perplexity of each method on WikiText2, PTB and C4 for Llama-2 70B. At this scale, both SpectrumQuant variants (with ${\bm Q}=\bm{UH}$ and ${\bm Q}=\bm{UB}$) match or improve on DartQuant in average perplexity, while using less than one-sixth of its calibration memory and finishing calibration in under 43 minutes on one GPU with pre-saved activations.}
\centering
\begin{tabular}{clcccc}
\toprule
\multicolumn{1}{c}{{Bits}} & \multicolumn{1}{c}{\multirow{2}{*}{{Method}}} & \multirow{2}{*}{{Wiki}} & \multirow{2}{*}{{PTB}} & \multirow{2}{*}{{C4}} & \multirow{2}{*}{{Avg $\downarrow$}} \\
\multicolumn{1}{c}{(W-A-KV)} & & & & & \\ \midrule
16 & Full Precision & 3.32 & 24.25 & 5.71 & 11.09 \\ \midrule
\multirow{4}{*}{4-4-4}
 & QuaRot & $3.802_{\pm 0.008}$ & $25.10_{\pm 0.44}$ & $6.126_{\pm 0.008}$ & $11.68_{\pm 0.14}$ \\
 & SpinQuant & $3.739_{\pm 0.016}$ & $26.92_{\pm 0.39}$ & $6.107_{\pm 0.026}$ & $12.26_{\pm 0.14}$ \\
 & DartQuant & $3.602_{\pm 0.004}$ & $26.39_{\pm 0.24}$ & $6.002_{\pm 0.007}$ & $12.00_{\pm 0.08}$ \\
 \rowcolor{tablegray} & SpectrumQuant ($\bm{Q}=\bm{UH}$) & $3.639_{\pm 0.035}$ & $26.34_{\pm 0.34}$ & $6.019_{\pm 0.013}$ & $12.00_{\pm 0.11}$ \\
 \rowcolor{tablegray} & SpectrumQuant ($\bm{Q}=\bm{UB}$) & $3.686_{\pm 0.020}$ & $25.522_{\pm 0.321}$ & $6.019_{\pm 0.012}$ & $11.74_{\pm 0.11}$ \\ \bottomrule
\end{tabular}
\end{table}

\begin{table}[H]
\centering
\caption{Accuracy on each of the eight common-sense benchmarks for Llama-2 70B, evaluated without in-context examples. SpectrumQuant (with ${\bm Q}=\bm{UH}$) obtains the highest average accuracy of all approaches, statistically tied with DartQuant, with calibration completed on a single GPU in under 43 minutes with pre-saved activations.}
\resizebox{\textwidth}{!}{
\begin{tabular}{clccccccccc}
\toprule
\multicolumn{1}{c}{Bits} & \multicolumn{1}{c}{\multirow{2}{*}{Method}} & \multirow{2}{*}{WG} & \multirow{2}{*}{SIQA} & \multirow{2}{*}{PIQA} & \multirow{2}{*}{OBQA} & \multirow{2}{*}{LAMB} & \multirow{2}{*}{HS} & \multirow{2}{*}{ARC-E} & \multirow{2}{*}{ARC-C} & \multirow{2}{*}{Avg $\uparrow$} \\
\multicolumn{1}{c}{(W-A-KV)} & & & & & & & & & & \\ \midrule
16 & Full Precision & 77.98 & 49.13 & 82.75 & 48.80 & 79.58 & 83.80 & 81.02 & 57.51 & 70.07 \\ \midrule
\multirow{4}{*}{4-4-4}
 & QuaRot & $77.38_{\pm 0.57}$ & $47.94_{\pm 0.57}$ & $81.70_{\pm 0.48}$ & $47.76_{\pm 0.62}$ & $78.65_{\pm 0.56}$ & $82.30_{\pm 0.23}$ & $78.88_{\pm 0.82}$ & $54.66_{\pm 0.68}$ & $68.66_{\pm 0.36}$ \\
 & SpinQuant & $77.58_{\pm 0.48}$ & $48.55_{\pm 0.88}$ & $82.01_{\pm 0.49}$ & $46.47_{\pm 0.81}$ & $78.70_{\pm 0.25}$ & $82.54_{\pm 0.24}$ & $78.87_{\pm 0.99}$ & $53.90_{\pm 1.33}$ & $68.58_{\pm 0.50}$ \\
 & DartQuant & $77.35_{\pm 0.81}$ & $48.37_{\pm 0.46}$ & $81.70_{\pm 0.42}$ & $47.76_{\pm 0.93}$ & $78.96_{\pm 0.53}$ & $83.07_{\pm 0.12}$ & $78.68_{\pm 0.86}$ & $55.05_{\pm 1.17}$ & $68.87_{\pm 0.42}$ \\
 \rowcolor{tablegray} & SpectrumQuant ($\bm{Q}=\bm{UH}$) & $78.15_{\pm 0.86}$ & $48.43_{\pm 0.83}$ & $81.65_{\pm 0.63}$ & $47.04_{\pm 1.37}$ & $78.91_{\pm 0.55}$ & $82.64_{\pm 0.19}$ & $79.36_{\pm 0.92}$ & $55.32_{\pm 2.09}$ & $68.94_{\pm 0.54}$ \\
 \rowcolor{tablegray} & SpectrumQuant ($\bm{Q}=\bm{UB}$) & $77.70_{\pm 0.73}$ & $48.22_{\pm 0.53}$ & $81.88_{\pm 0.69}$ & $47.60_{\pm 1.03}$ & $78.96_{\pm 0.20}$ & $82.43_{\pm 0.21}$ & $78.40_{\pm 0.98}$ & $53.80_{\pm 1.19}$ & $68.62_{\pm 0.37}$ \\ \bottomrule
\end{tabular}
}
\end{table}

\section{Comparison of Layerwise Quantization Error}
\label{appendix:layerwise-quant-error}

\begin{figure}[H]
  \centering
  \includegraphics[trim={0cm 2cm 0cm 0cm}, clip, width=\linewidth]{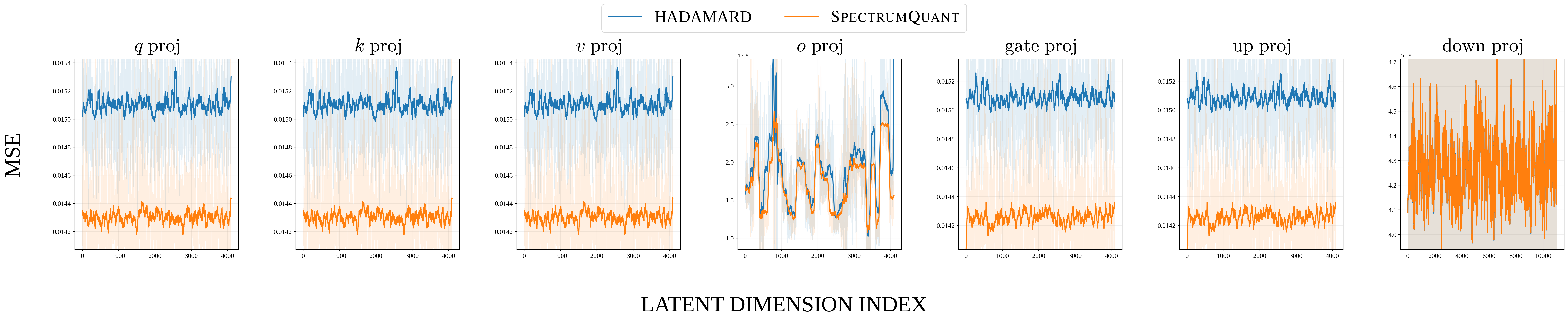}
  \includegraphics[trim={0cm 2cm 0cm 2.2cm}, clip, width=\linewidth]{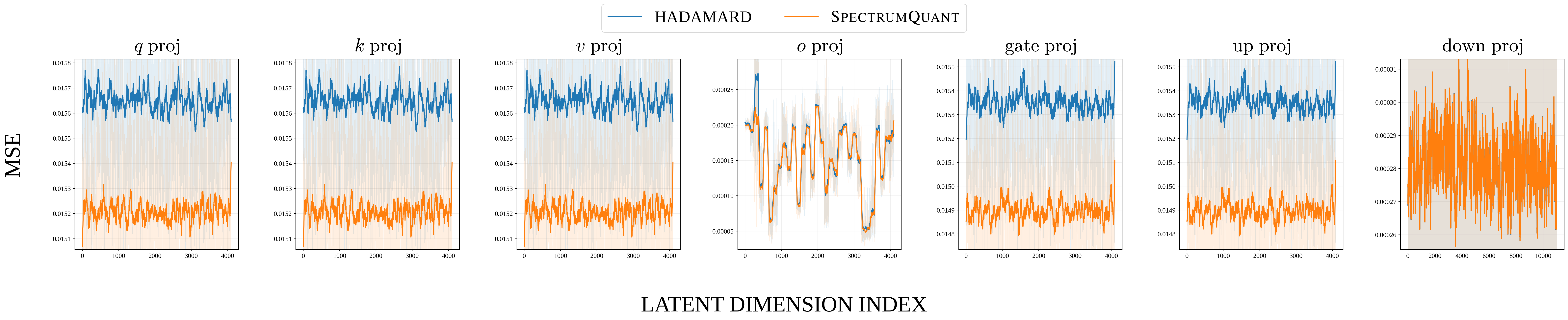}
    \includegraphics[trim={0cm 2cm 0cm 2.2cm}, clip, width=\linewidth]{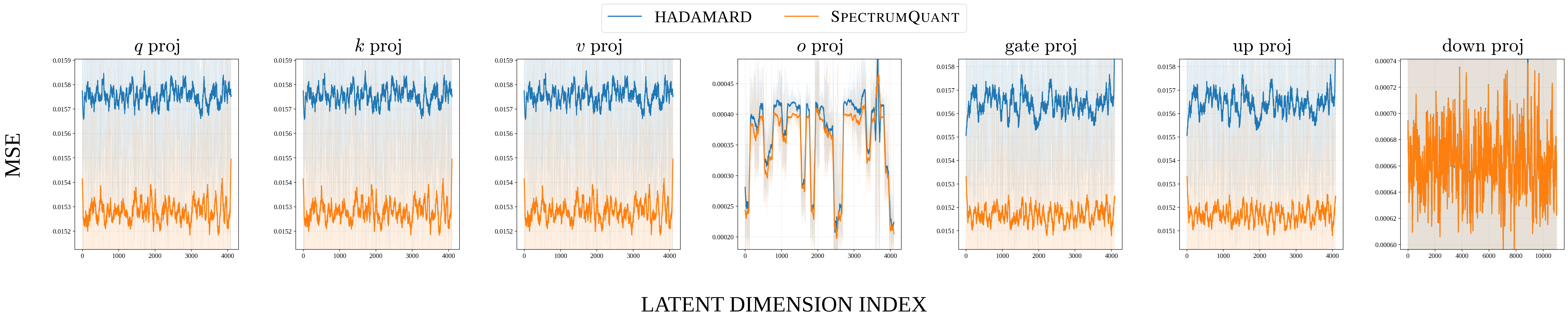}
    \includegraphics[trim={0cm 2cm 0cm 2.2cm}, clip, width=\linewidth]{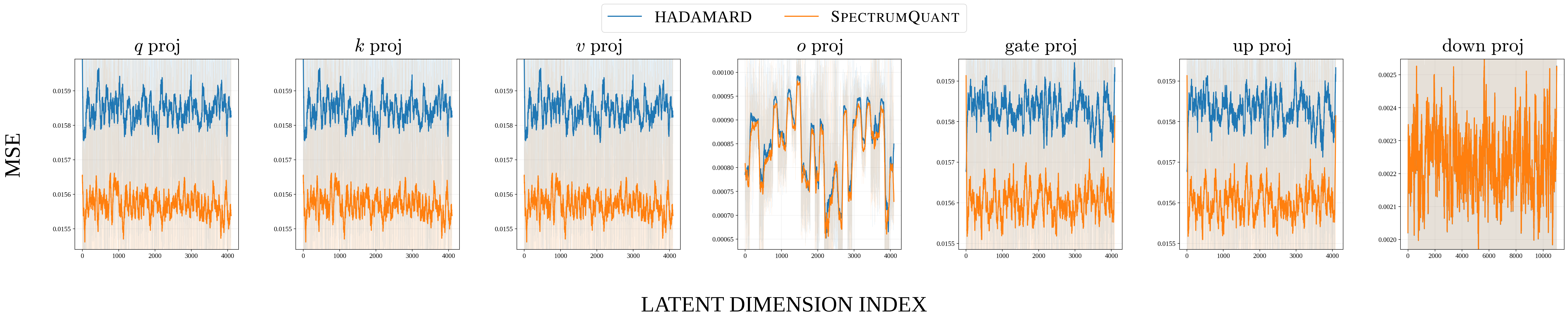}
    \includegraphics[trim={0cm 2cm 0cm 2.2cm}, clip, width=\linewidth]{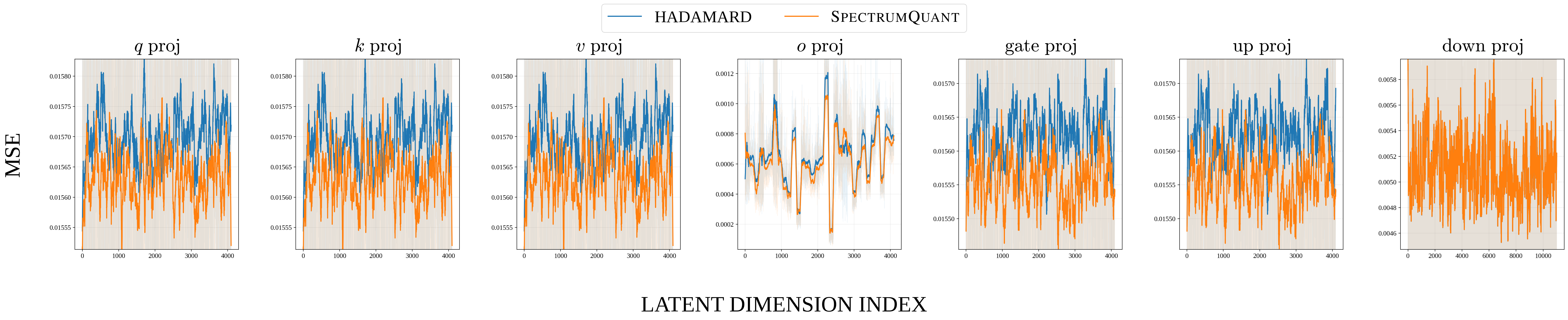}
    \includegraphics[trim={0cm 2cm 0cm 2.2cm}, clip, width=\linewidth]{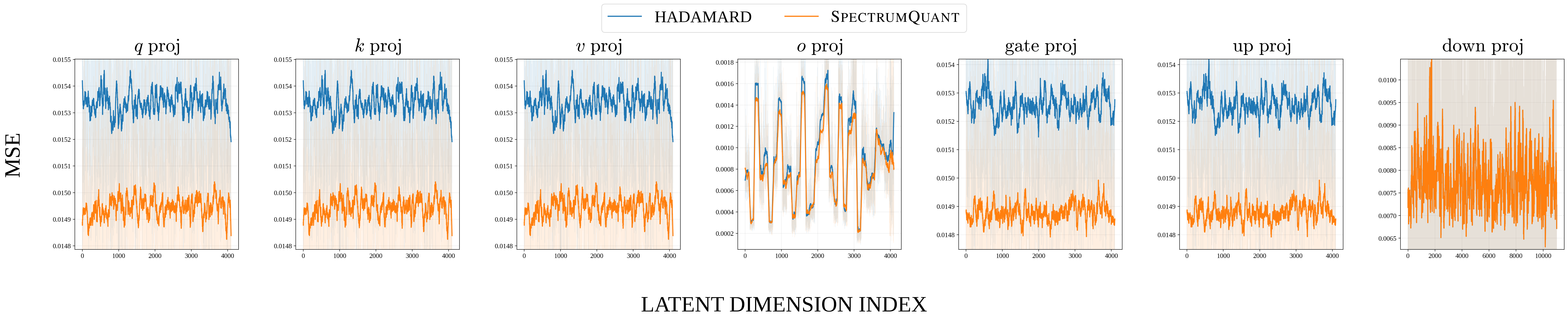}
  \includegraphics[trim={0cm 0cm 0cm, 2.2cm}, clip, width=\linewidth]{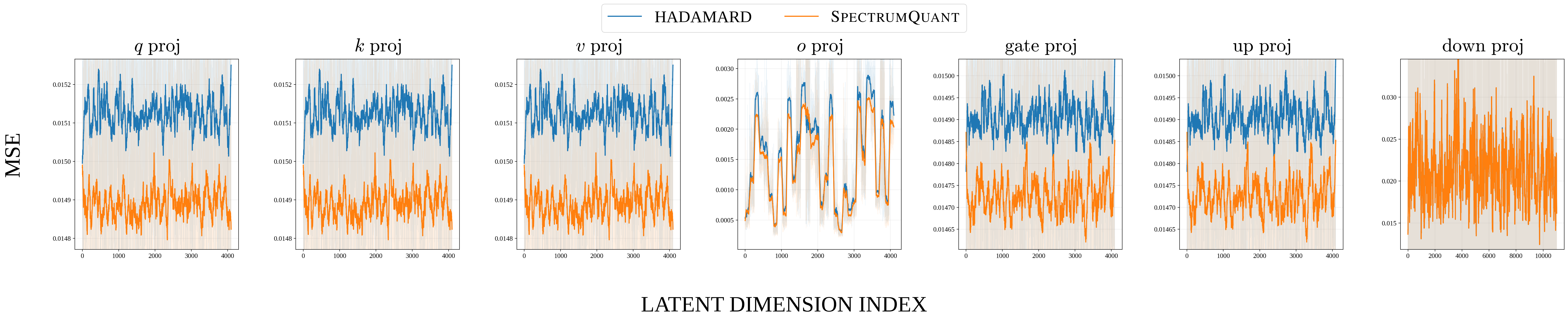}
  \caption{Comparison of sorted per-channel MSE across hidden dimensions for different projection activations in Llama-2 7B at layers 2,5,8,15,21,27,30 from the top to the bottom.}
  \label{fig:layerwise_quant_loss}
\end{figure}

\end{document}